\documentclass[review,11pt]{ReportTemplate}

\usepackage{bm}
\usepackage{graphicx}
\usepackage{paralist,algorithmic,algorithm}
\usepackage{amsmath,mathrsfs,amssymb}

\newenvironment{proof}{\textbf{Proof:}\ }{\hspace{\stretch{1}}$\square$\\}

\usepackage{url, graphicx, subfigure}

\newtheorem{thm}{Theorem}
\newtheorem{lemma}{Lemma}

\def \x {\mathbf{x}}

\def \w {\mathbf{w}}

\begin{document}

\begin{frontmatter}
\title{Fast Multi-Instance Multi-Label Learning}

\author{Sheng-Jun Huang}
\author{Wei Gao}
\author{Zhi-Hua Zhou\corref{cor1}}
\address{National Key Laboratory for Novel Software Technology\\
Nanjing University, Nanjing 210093, China} \cortext[cor1]{\small Corresponding author.
Email: zhouzh@nju.edu.cn}

\begin{abstract}
In many real-world tasks, particularly those involving data objects with complicated semantics such as images and texts, one object can be represented by multiple instances and simultaneously be associated with multiple labels. Such tasks can be formulated as multi-instance multi-label learning (MIML) problems, and have been extensively studied during the past few years. Existing MIML approaches have been found useful in many applications; however, most of them can only handle moderate-sized data. To efficiently handle large data sets, in this paper we propose the MIMLfast approach, which first constructs a low-dimensional subspace shared by all labels, and then trains label specific linear models to optimize approximated ranking loss via stochastic gradient descent. Although the MIML problem is complicated, MIMLfast is able to achieve excellent performance by exploiting label relations with shared space and discovering sub-concepts for complicated labels. Experiments show that the performance of MIMLfast is highly competitive to state-of-the-art techniques, whereas its time cost is much less; particularly, on a data set with 20K bags and 180K instances, MIMLfast is more than 100 times faster than existing MIML approaches. On a larger data set where none of existing approaches can return results in 24 hours, MIMLfast takes only 12 minutes. Moreover, our approach is able to identify the most representative instance for each label, and thus providing a chance to understand the relation between input patterns and output label semantics.
\end{abstract}

\begin{keyword}
MIML, multi-instance multi-label learning, fast, key instance, sub-concepts
\end{keyword}

\end{frontmatter}

\section{Introduction}
In traditional supervised learning, one object is represented by a single instance and associated with only one label. However, in many real world applications, one object can be naturally decomposed into multiple instances, and has multiple class labels simultaneously. For example, in image classification problems, an image usually contains multiple objects, and can be divided into several segments, where each segment is represented with an instance, and corresponds to a semantic label \cite{ZZ07}; in text categorization tasks, an article may belong to multiple categories, and can be represented by a bag of instances, one for a paragraph \cite{YZH09}; in gene function prediction tasks, a gene usually has multiple labels since it is related to multiple functions, and can be represented with a set of images with different views \cite{LJKYZ}. Multi-instance multi-label learning (MIML) is a recent proposed framework for such complicated objects \cite{ZZHL12}.

During the past years, many MIML algorithms were proposed \cite{ZZ07,ZHMWQW08,JWZ09,YZH09,ZW09,LO10,N10,Z10,BFR12,ZZHL12}. They achieved decent performances and validated the superiority of MIML framework in different applications. However, along with the enhancing of expressive power, the hypothesis space of MIML expands dramatically, resulting in the high complexity and low efficiency of existing approaches. These approaches are usually time-consuming, and cannot handle large scale data, thus strongly limit the application of multi-instance multi-label learning.

In this paper, we propose a novel approach MIMLfast to learn on multi-instance multi-label data fast. Though simple linear models are employed for efficiency, MIMLfast provides an effective approximation of the original MIML problem. Specifically, to utilize the relations among multiple labels, we first learn a shared space for all the labels from the original features, and then train label specific linear models from the shared space. To identify the key instance to represent a bag for a specific label, we train the classification model on the instance level, and then select the instance with maximum prediction. To make the learning efficient, we employ stochastic gradient descent (SGD) to optimize an approximated ranking loss. At each step of SGD, MIMLfast randomly samples a triplet which consists of a bag, a relevant label of the bag and an irrelevant label, and optimizes the model to rank the relevant label before the irrelevant one if such an order is violated.

While most existing approaches focus on improving generalization, another important task of MIML learning is to understand the relation between input patterns and output label semantics \cite{LHJZ12}. Our approach can naturally identify the most representative instance for each label. In addition, we propose to discover sub-concepts for complicated labels, which frequently occur in MIML tasks.

The rest of the paper is organized as follows. We propose the MIMLfast approach in Section 2, and then present the experiments in Section 3. Section 4 reviews some related work, followed by the conclusion in Section 5.

\section{The MIMLfast Approach}
We denote by $\{(X_1,Y_1),(X_2,Y_2),\cdots,(X_n,Y_n)\}$ the training data that consists of $n$ examples, where each bag $X_i$ has $z_i$ instances $\{\x_{i,1},\x_{i,2},\cdots,\x_{i,z_i}\}$ and $Y_i$ contains the labels associated with $X_i$, which is a subset of all possible labels $\{y_1, y_2\cdots y_L\}$.

We first discuss on how to build the classification model on the instance level, and then try to get the labels of bags from instance predictions. To handle a problem with multiple labels, the simplest way is to degenerate it into a series of single label problems by training one model for each label independently. However, such a degenerating approach may lose information since it treats the labels independently and ignores the relations among them. In our approach, we formulate the model as a combination of two components. The first component learns a linear mapping from the original feature space to a low dimensional space, which is shared by all the labels. Then the second component learns label specific models based on the shared space. The two components are optimized interactively to fit training examples from all labels. In such a way, examples from each label will contribute the optimization of the shared space, and labels with strong relations are expected to help each other. Formally, given an instance $\x$, we define the classification model on label $l$ as
$$f_l(\x)=\w_l^\top W_0\x,$$
where $W_0$ is a $m\times d$ matrix which maps the original feature vectors to the shared space, and $\w_l$ is the $m$-dimensional weight vector for label $l$. $d$ and $m$ are the dimensionalities of the feature space and the shared space, respectively.

Objects in multi-instance multi-label learning tasks usually have complicated semantic; and thus examples with diverse contents may be assigned the same label. For example, the content of an image labeled \emph{apple} can be a mobile phone, a laptop or just a real apple. It is difficult to train a single model to classify images with such diverse contents into the same category. Instead, we propose to learn multiple models for a complicated label, one for a sub-concept, and automatically decide which sub-concept one example belongs to. The model of each sub-concept is much simpler and may be more easily trained to fit the data. We assume that there are $K$ sub-concepts for each label. For a given example with label $l$, the sub-concept it belongs to is automatically determined by first examining the prediction values of the $K$ models, and then selecting the sub-concept with maximum prediction value. Now we can redefine the prediction of instance $\x$ on label $l$ as:
\begin{equation}\label{eq:model}
f_l(\x)=\max_{k=1\cdots K} f_{l,k}(\x)=\max_{k=1\cdots K} \w_{l,k}^\top W_0\x,
\end{equation}
where $\w_{l,k}$ corresponds to the $k$-th sub-concept of label $l$. Note that although we assume there are $K$ sub-concepts for each label, empty sub-concepts are allowed, i.e., examples of a simple label may be distributed in only a few or even one sub-concept.

We then look at how to get the predictions of bags from the instance level models. It is usually assumed that a bag is positive if and only if it contains at least one positive instance \cite{DLL97,BFR12}. Under this assumption, the prediction of a bag $X$ on label $l$ can be defined as the maximum of predictions of all instances in this bag:
$$f_l(X)=\max_{\x\in X}f_l(\x).$$
We call the instance with maximum prediction the key instance of $X$ on label $l$.

With the above model, for an example $X$ and one of its relevant labels $l$, we define $R(X, l)$ as
\begin{equation}\label{eq:RL}
R(X, l)=\sum_{j\in \bar{Y}} I[f_j(X)>f_l(X)],
\end{equation}
where $\bar{Y}$ denotes the set of irrelevant labels of $X$, and $I[\cdot]$ is the indicator function which returns $1$ if the argument is true and $0$ otherwise. Essentially, $R(X, l)$ counts how many irrelevant labels are ranked before label $l$ on the bag $X$.

Based on $R(X, l)$, we further define the ranking error \cite{UBG09} with respect to an example $X$ on label $l$ as
\begin{equation}\label{eq:error}
\epsilon(X, l)=\sum_{i=1}^{R(X, l)}\frac{1}{i}.
\end{equation}
It is obvious that the ranking error $\epsilon$ would be larger for lower $l$ being ranked. Finally, we have the ranking error on the whole dataset as
$$
\text{Rank Error} = \sum_{i=1}^n \sum_{l\in Y_i}\epsilon(X, l).
$$

Based on Eq.~\ref{eq:RL}, the ranking error $\epsilon(X, l)$ can be spread into all irrelevant labels in $\bar{Y}$ as:
\begin{equation}\label{eq:errorspread}
\epsilon(X, l)=\sum_{j\in \bar{Y}} \epsilon(X, l)\frac{I[f_j(X)>f_l(X)]}{R(X, l)}.
\end{equation}
Due to non-convexity and discontinuousness, it is rather difficult to optimize the above equation directly because such optimization often leads to NP-hard problems. We instead explore the following hinge loss, which has been shown as an optimal choice among all convex surrogate losses \cite{BDLSS12},
\begin{equation}
\Psi(X, l)=\sum_{j\in\bar{Y}} \epsilon(X, l)\frac{|1+f_j(X)-f_l(X)|_+}{R(X, l)},
\end{equation}
where $|q|_+=q$ if $q\geq0$; otherwise, $|q|_+=0$. The surrogate loss $\Psi(X, l)$ can be viewed as an upper bound of $\epsilon(X, l)$ with the following lemma:
\begin{lemma}
$\epsilon(X, l)\leq\Psi(X, l)$.
\end{lemma}
\begin{proof}
This lemma holds from $I[q]\leq |1-q|_+$.
\end{proof}

We then employ stochastic gradient descent (SGD) \cite{RM51} to minimize the ranking error. At each iteration of SGD, we randomly sample a bag $X$, one of its relevant labels $y$, and one of its irrelevant labels $\bar{y}\in\bar{Y}$ to form a triplet $(X, y, \bar{y})$, which will induce a loss:
\begin{equation}
\mathcal{L}(X,y,\bar{y})=\epsilon(X, y)|1+f_{\bar{y}}(X)-f_y(X)|_+.
\end{equation}
We then have the following lemma to disclose the relation between $\Psi(X, y)$ and $\mathcal{L}(X,y,\bar{y})$.
\begin{lemma}
$\Psi(X, y)=E_{\bar{y}}[\mathcal{L}(X,y,\bar{y})]$,
where $E[\cdot]$ denotes the expectation on the uniform distribution over $\bar{Y}$.
\end{lemma}
\begin{proof} This lemma follows from the fact that probability of randomly choosing $\bar{y}$ in $\bar{Y}$ is $1/R(X, y)$.
\end{proof}

To minimize $\mathcal{L}(X,y,\bar{y})$, it is required to calculate $R(X, y)$ in advance, i.e., we have to compare $f_y(X)$ with $f_{\bar{y}}(X)$ for each $\bar{y}\in\bar{Y}$, whereas this could be time consuming when the number of possible labels is large. Therefore, we use an approximation to estimate $R(X, y)$ in our implementation, inspired by Weston et al. \cite{WBU11}. Specifically, at each SGD iteration, we randomly sample labels from the irrelevant label set $\bar{Y}$ one by one, until a violated label $\bar{y}$ occurs. Here we call $\bar{y}$ a violated label if it was ranked before $y$, i.e., $f_{\bar{y}}(X)>f_y(X)-1$. Without loss of generality, we assume that the first violated label is found at the $v$-th sampling step, and then, $R(X, y)$ can be approximated by $\lfloor{|\bar{Y}|}/{v}\rfloor$ with the following lemma:
\begin{lemma} We denote by $\xi$ a random event with $\xi=i$ representing the event that first violated label is at the $i$-th sampling step. We have
\[
\frac{R(X, y)}{|\bar{Y}|}\approx E_\xi\left[\frac{1}{\xi}\right].
\]
\end{lemma}
\begin{proof}
For convenience, we set $p={R(X, y)}/{|\bar{Y}|}$ and assume $0<p<1$ without loss of generality. It is easy to derive the probability
\[
\Pr[\xi=i]=(1-p)^{i-1}p \text{ for }i\geq1,
\]
and we further have
\begin{eqnarray*}
% \nonumber to remove numbering (before each equation)
E_\xi\left[\frac{1}{\xi}\right]&=&\sum_{i=1}^\infty\frac{1}{i} p(1-p)^{i-1} = \frac{p}{1-p}\sum_{i=1}^\infty\frac{1}{i} (1-p)^{i}\\
&=&\frac{-p}{1-p}\ln(1-(1-p))\approx p
\end{eqnarray*}
where we use $\sum_{i=1}^\infty\frac{1}{i} (1-p)^{i}=-\ln(p)$ and $\ln(1+q)\approx q$. This completes the proof.
\end{proof}

We assume that the triplet sampled at the $t$-th SGD iteration is $(X, y, \bar{y})$, on label $y$, the key instance is $\x$, and achieves the maximum prediction on the $k$-th sub-concept, while on label $\bar{y}$, the instance $\bar{\x}$ achieves the maximum prediction on the $\bar{k}$-th sub-concept. Then we have the approximated ranking loss for the triplet:
\begin{align*}
&\mathcal{L}(X,y,\bar{y})=\epsilon(X, y)|1+f_{\bar{y}}(X)-f_y(X)|_+\\
\approx &\left\{
\begin{aligned}
&0\qquad\qquad\qquad\qquad\qquad\qquad\quad\qquad\text{if $\bar{y}$ is not violated;}&\\
&S_{\bar{Y},v}(1+[\w_{\bar{y},\bar{k}}^t]^\top W_0^t\bar{\x}-[\w_{y,k}^t]^\top W_0^t\x)\quad\qquad\text{otherwise.}&
\end{aligned}
\right.
\end{align*}
Here we introduce $S_{\bar{Y},v}=\sum\nolimits_{i=1}^{\lfloor\frac{|\bar{Y}|}{v}\rfloor}\frac{1}{i}$ for the convenience of presentation. So, if a violated label $\bar{y}$ is sampled, we perform the gradient descent on the three parameters according to:
\begin{align}
W_0^{t+1}&=W_0^t-\gamma_t S_{\bar{Y},v}(\w_{\bar{y},\bar{k}}^t\bar{\x}^\top-\w_{y,k}^t\x^\top)\label{eq:sgdw0}\\
\w_{y,k}^{t+1}&=\w_{y,k}^t+\gamma_t S_{\bar{Y},v}W_0^t\x\label{eq:sgdwy}\\
\w_{\bar{y},\bar{k}}^{t+1}&=\w_{\bar{y},\bar{k}}^t-\gamma_t S_{\bar{Y},v}W_0^t\bar{\x}\label{eq:sgdwybar}
\end{align}
where $\gamma_t$ is the step size of SGD at the $t$-th iteration. After the update of the parameters, $\w_{y,k}$, $\w_{\bar y, \bar k}$ and each column of $W_0$ are normalized to have a L2 norm smaller than a constant $C$.

\begin{algorithm}[!t]
   \caption{The MIMLfast algorithm}
   \label{alg}
\begin{algorithmic}[1]
   \STATE {\emph{INPUT:}} \quad training data, parameters $m$, $C$, $K$ and $\gamma_t$
   \STATE {\emph{TRAIN:}}
   \STATE \quad initialize $W_0$ and $\w_{l,k}\ (l=1\cdots L, k=1\cdots K)$
   \STATE \quad {\bfseries repeat}:
        \STATE \qquad randomly sample a bag $X$ and one of its label $y$
        \STATE \qquad $(\x,k)=\arg\max_{\x\in X, k\in\{1\cdots K\}}f_{y,k}(\x)$
        \STATE \qquad {\bfseries for $i=1:|\bar{Y}|$}
        \STATE \qquad\quad sample an irrelevant label $\bar{y}$ from $\bar{Y}$
        \STATE \qquad\quad $(\bar{\x},\bar{k})=\arg\max_{\x\in X, \bar{k}\in\{1\cdots K\}}f_{\bar{y},\bar{k}}(\x)$
        \STATE \qquad\quad {\bfseries if $f_{\bar{y}}(X)>f_y(X)-1$}
        \STATE \qquad\qquad $v=i$
        \STATE \qquad\qquad update $W_0$, $\w_{y,k}$ and $\w_{\bar{y},\bar{k}}$ as Eqs. \ref{eq:sgdw0} to \ref{eq:sgdwybar}
        \STATE \qquad\qquad normalize $W_0$, $\w_{y,k}$ and $\w_{\bar{y},\bar{k}}$
        \STATE \qquad\qquad break
        \STATE \qquad\quad {\bfseries end if}
        \STATE \qquad {\bfseries end for}
   \STATE \quad {\bfseries until} stop criterion reached
   \STATE {\emph{TEST:}}
   \STATE \quad Relevant labels set for the test bag $X_{\text{test}}$ is: $\{l|1+f_l(X_{\text{test}})>f_{\hat{y}}(X_{\text{test}})\}$
\end{algorithmic}
\end{algorithm}

The pseudo code of MIMLfast is presented in Algorithm~\ref{alg}. First, each column of $W_0$ and $\w_l^k$ for all labels $l$ and all sub-concepts $k$ are initialized at random with mean 0 and standard deviation $1/\sqrt{d}$. Then at each iteration of SGD, a triplet $(X, y, \bar{y})$ is randomly sampled, and their corresponding key instance and sub-concepts are identified. After that, gradient descent is performed to update the three parameters: $W_0$, $\w_{y,k}$ and $\w_{\bar{y},\bar{k}}$ according to Eqs. \ref{eq:sgdw0} to \ref{eq:sgdwybar}. At last, the updated parameters are normalized such that their norms will be upper bounded by $C$. This procedure is repeated until some stop criteria reached. In our experiments, we sample a small subset from the training data to form a validation set, and stop the training if the ranking loss does not decrease anymore on the validation set.

We then present some theoretical guarantees on the convergence rate of the optimization. Denoting by
$$
\mathcal{L}_t(W_0,\w_{y,k},\w_{\bar{y},\bar{k}}) =S_{\bar{Y},v}(1+\w_{\bar{y},\bar{k}}^\top W_0\bar{\x}_t-\w_{y,k}^\top W_0\x_t)_+
$$
the loss of $t$-th SGD iteration with model parameters $W_0$, $\w_{y,k}$ $\w_{\bar{y},\bar{k}}$, and
$$
(W_0^*,\w^*_{l,k})\in\arg\min\sum_t \mathcal{L}_t(W_0,\w_{y,k},\w_{\bar{y},\bar{k}})
$$
the optimal solution, we have:
\begin{thm}
Suppose $\|\x_t\|\leq 1$, $\|W^t_0\|\leq C\sqrt{d}$, $\|\w^t_{y,k}\|\leq C$ and $\|\w^t_{\bar{y},\bar{k}}\|\leq C$. By choosing proper $W_0^0$, $\w^0_{y,k}$ and $\gamma_t$, it holds that
\[
\sum_t^T\mathcal{L}_t(W_0^t,\w^t_{y,k},\w^t_{\bar{y},\bar{k}})-\sum_t^T \mathcal{L}_t(W_0^*,\w^*_{y,k},\w^*_{\bar{y},\bar{k}})\leq B\sqrt{T}
\]
where $B=4+(d+2\sqrt{d})C^2\sum_{i=1}^L1/i$.
\end{thm}
\begin{proof}
We present the main steps due to space limitation. Since the function $\mathcal{L}_t(W_0^t,\w^t_{y,k},\w^t_{\bar{y},\bar{k}})$ is convex with respect to $\w^t_{y,k}$ and $\w^t_{\bar{y},\bar{k}}$, we have
\begin{eqnarray*}
\lefteqn{\mathcal{L}_t(W_0^t,\w^t_{y,k},\w^t_{\bar{y},\bar{k}})- \mathcal{L}_t(W_0^t,\w^*_{y,k},\w^*_{\bar{y},\bar{k}})}\nonumber\\
&\leq& [\partial\mathcal{L}_t(W_0^t,\w^t_{y,k},\w^t_{\bar{y},\bar{k}})/\partial \w^t_{y,k}]^\top (\w^t_{y,k}-\w^*_{y,k})\nonumber \\
&&+[\partial\mathcal{L}_t(W_0^t,\w^t_{y,k},\w^t_{\bar{y},\bar{k}})/\partial \w^t_{\bar{y},\bar{k}}]^\top (\w^t_{\bar{y},\bar{k}}-\w^*_{\bar{y},\bar{k}}).
\end{eqnarray*}
From Eqs.~\eqref{eq:sgdwy} and \eqref{eq:sgdwybar}, we have
\begin{multline*}
\|\w^{t+1}_{y,k}-\w^*_{y,k}\|\leq \|\w^{t}_{y,k}-\w^*_{y,k}\| + \Delta_t -2\gamma_t[\partial\mathcal{L}_t(W_0^t,\w^t_{y,k},\w^t_{\bar{y},\bar{k}})/\partial \w^t_{y,k}]^\top(\w^{t}_{y,k}-\w^*_{y,k})
\end{multline*}
where
\[
\Delta_t=\gamma_t^2\Big\|S_{\bar{Y},v} W_0^t\Big\|^2 \leq dC^2\gamma_t^2\sum\nolimits_{i=1}^L1/i.
\]
This follows that
\begin{align*}
\|\w^{t+1}_{y,k}-\w^*_{y,k}\|\leq\ & \|\w^{t}_{y,k}-\w^*_{y,k}\| + dC^2\gamma_t^2\sum\nolimits_{i=1}^L1/i\\ &-2\gamma_t[\partial\mathcal{L}_t(W_0^t,\w^t_{y,k},\w^t_{\bar{y},\bar{k}})/\partial \w^t_{y,k}]^\top(\w^{t}_{y,k}-\w^*_{y,k}).
\end{align*}
In a similar manner, we have
\begin{align*}
\|\w^{t+1}_{\bar{y},\bar{k}}-\w^*_{\bar{y},\bar{k}}\|\leq\ & \|\w^{t}_{\bar{y},\bar{k}}-\w^*_{\bar{y},\bar{k}}\| + dC^2\gamma_t^2\sum\nolimits_{i=1}^L1/i\\ &-2\gamma_t[\partial\mathcal{L}_t(W_0^t,\w^t_{y,k},\w^t_{\bar{y},\bar{k}})/\partial \w^t_{\bar{y},\bar{k}}]^\top(\w^{t}_{\bar{y},\bar{k}}-\w^*_{\bar{y},\bar{k}}).
\end{align*}
Summing over $t=0,...,T-1$, and by setting $\gamma_t=1/\sqrt{t}$ and simple calculation, we have
\begin{align*}
&\sum_{t=1}^{T-1}\mathcal{L}_t(W_0^t,\w^t_{y,k},\w^t_{\bar{y},\bar{k}})- \sum_{t=1}^{T-1}\mathcal{L}_t(W_0^t,\w^*_{y,k},\w^*_{\bar{y},\bar{k}})\\
\leq& \frac{2}{\gamma_T}+B\sum_{t=1}^{T-1} \frac{\gamma_t}{2}\leq (2+dC^2\sum\nolimits_{i=1}^L1/i)\sqrt{T}.
\end{align*}
Further, we have
\begin{eqnarray*}
% \nonumber to remove numbering (before each equation)
\lefteqn{\sum_{t=1}^{T-1}\mathcal{L}_t(W_0^t,\w^*_{y,k},\w^*_{\bar{y},\bar{k}})- \sum_{t=1}^{T-1}\mathcal{L}_t(W_0^*,\w^*_{y,k},\w^*_{\bar{y},\bar{k}})}\\
&=&\sum_{t=1}^{\sqrt{T}-1}\mathcal{L}_t(W_0^t,\w^*_{y,k}, \w^*_{\bar{y},\bar{k}})- \sum_{t=1}^{\sqrt{T}-1}\mathcal{L}_t(W_0^*,\w^*_{y,k},\w^*_{\bar{y},\bar{k}})\\
&&+\sum_{t=\sqrt{T}}^{T-1}\mathcal{L}_t(W_0^t,\w^*_{y,k},\w^*_{\bar{y},\bar{k}})- \sum_{t=\sqrt{T}}^{T-1}\mathcal{L}_t(W_0^*,\w^*_{y,k},\w^*_{\bar{y},\bar{k}})\\
&\leq&2(1+\sqrt{d}C^2\sum_{i=1}^L1/i)\sqrt{T}
\end{eqnarray*}
by selecting proper initial values and simple calculation. This theorem follows as desired.
\end{proof}

In the test phase of the algorithm, for a bag $X_{\text{test}}$, we can get the prediction value on each label, and consequently the rank of all labels. For single label classification problem, it is very easy to get the label of $X_{\text{test}}$ by selecting the one with largest prediction value. However, in multi-label learning, the bag $X_{\text{test}}$ may have more than one label; and thus one do not know how many labels should be selected as relevant ones from the ranked label list \cite{FHLB08}. To solve this problem, we assign each bag a dummy label, denoted by $\hat{y}$, and train the model to rank the dummy label before all irrelevant labels while after the relevant ones. To implement this idea, we pay a special consideration on constructing the irrelevant labels set $\bar{Y}$. Specifically, when $X$ and its label $y$ are sampled (in Line 6 of Algorithm \ref{alg}), the algorithm will first examine whether $y$ is the dummy label. If $y=\hat{y}$, then $\bar{Y}$ consists of all the irrelevant labels; otherwise, $\bar{Y}$ contains both the dummy label and all the irrelevant labels. In such a way, the model will be trained to rank the dummy label between relevant labels and irrelevant ones. For a test bag, the labels with larger prediction value than that on the dummy label are selected as relevant labels.

\begin{table}[!t]
\caption{Experimental data sets (6 moderate size and 2 large size)}
\label{table:data}
\begin{center}%\small
\begin{tabular}{l|ccccc}
\hline
Data & \# ins. & \# bag & \# label & \# label per bag\\
\hline
\emph{Letter Frost}\raisebox{1em}{} & 565 & 144 & 26 & 3.6\\
\emph{Letter Carroll} & 717 & 166 & 26 & 3.9\\
\emph{MSRC v2} & 1758 & 591 & 23 & 2.5\\
\emph{Reuters} & 7119 & 2000 & 7 & 1.2\\
\emph{Bird Song} & 10232 & 548 & 13 & 2.1\\
\emph{Scene} & 18000 & 2000 & 5 & 1.2\\
\hline
\emph{Corel5K}\raisebox{1em}{} & 47,065 & 5,000 & 260 & 3.4\\
\emph{MSRA} & 270,000 & 30,000 & 99 & 2.7\\
\hline
\end{tabular}
\end{center}
\vspace{-0.2cm}
\end{table}

\section{Experiments}
\subsection{Settings}
We compare MIMLfast with six state-of-the-art MIML methods: DBA \cite{YZH09}, a generative model for MIML learning; KISAR \cite{LHJZ12}, a MIML algorithm tries to discover instance-label relation; MIMLBoost \cite{ZZ07}, a boosting method decomposes MIML into multi-instance single label problems; MIMLkNN \cite{Z10}, a MIML nearest neighbor algorithm; MIMLSVM \cite{ZZ07}, a SVM style algorithm which decomposes MIML into single instance multi-label problems; and RankLossSIM \cite{BFR12}, a MIML algorithm minimizes ranking loss for instance annotation.

\begin{table}[!t]
\caption{Comparison results (mean$\pm$std.) on moderate-sized data sets. $\uparrow$($\downarrow$) indicates that the larger (smaller) the value, the better the performance; $\bullet$($\circ$) indicates that MIMLfast is significantly better(worse) than the corresponding method based on paired $t$-tests at $95\%$ significance level; N/A indicates that no result was obtained in 24 hours.}
\begin{center}\label{table:result}\fontsize{8.6pt}{8.8pt}\selectfont
\begin{tabular}{l|c|ccccccc}
\hline
 \raisebox{1em}{$\ $}& MIMLfast & DBA&KISAR&MIMLBoost&MIMLkNN&MIMLSVM&RankL.SIM\\
\hline
\multicolumn{5}{l}{\emph{Letter Carroll}}\raisebox{0.5em}{$\ $}\\
\hline
\textsf{h.l.} $^{\downarrow}$ & .134$\pm$.012$\ $ & .180$\pm$.010$\bullet$ & .150$\pm$.008$\bullet$ & .153$\pm$.008$\bullet$ & .170$\pm$.017$\bullet$ & .154$\pm$.007$\bullet$ & .132$\pm$.006$\ $\\
\textsf{o.e.} $^{\downarrow}$ & .119$\pm$.050$\ $ & .248$\pm$.036$\bullet$ & .058$\pm$.096$\circ$ & .645$\pm$.062$\bullet$ & .312$\pm$.043$\bullet$ & .554$\pm$.043$\bullet$ & .167$\pm$.050$\bullet$\\
\textsf{co.} $^{\downarrow}$ & .380$\pm$.029$\ $ & .909$\pm$.023$\bullet$ & .870$\pm$.018$\bullet$ & .730$\pm$.039$\bullet$ & .460$\pm$.030$\bullet$ & .905$\pm$.020$\bullet$ & .389$\pm$.037$\ $\\
\textsf{r.l.} $^{\downarrow}$ & .130$\pm$.013$\ $ & .622$\pm$.033$\bullet$ & .873$\pm$.043$\bullet$ & .477$\pm$.035$\bullet$ & .194$\pm$.019$\bullet$ & .710$\pm$.029$\bullet$ & .134$\pm$.017$\ $\\
\textsf{a.p.} $^{\uparrow}$ & .715$\pm$.032$\ $ & .324$\pm$.029$\bullet$ & .181$\pm$.027$\bullet$ & .263$\pm$.020$\bullet$ & .611$\pm$.023$\bullet$ & .350$\pm$.022$\bullet$ & .708$\pm$.026$\ $\\

\hline
\multicolumn{5}{l}{\emph{Letter Frost}}\raisebox{0.5em}{$\ $}\\
\hline
\textsf{h.l.} $^{\downarrow}$ & .136$\pm$.014$\ $ & .166$\pm$.010$\bullet$ & .200$\pm$.013$\bullet$ & .139$\pm$.007$\ $ & .139$\pm$.010$\ $ & .154$\pm$.013$\bullet$ & .136$\pm$.010$\ $\\
\textsf{o.e.} $^{\downarrow}$ & .151$\pm$.041$\ $ & .228$\pm$.056$\bullet$ & .380$\pm$.064$\bullet$ & .257$\pm$.101$\bullet$ & .288$\pm$.077$\bullet$ & .581$\pm$.045$\bullet$ & .203$\pm$.055$\bullet$\\
\textsf{co.} $^{\downarrow}$ & .375$\pm$.042$\ $ & .857$\pm$.032$\bullet$ & .906$\pm$.019$\bullet$ & .728$\pm$.038$\bullet$ & .463$\pm$.035$\bullet$ & .884$\pm$.028$\bullet$ & .372$\pm$.038$\ $\\
\textsf{r.l.} $^{\downarrow}$ & .134$\pm$.019$\ $ & .580$\pm$.033$\bullet$ & .705$\pm$.036$\bullet$ & .478$\pm$.030$\bullet$ & .199$\pm$.018$\bullet$ & .810$\pm$.101$\bullet$ & .138$\pm$.019$\ $\\
\textsf{a.p.} $^{\uparrow}$ & .704$\pm$.034$\ $ & .358$\pm$.030$\bullet$ & .264$\pm$.028$\bullet$ & .235$\pm$.014$\bullet$ & .612$\pm$.027$\bullet$ & .226$\pm$.060$\bullet$ & .686$\pm$.035$\bullet$\\

\hline
\multicolumn{5}{l}{\emph{MSRC v2}}\raisebox{0.5em}{$\ $}\\
\hline
\textsf{h.l.} $^{\downarrow}$ & .100$\pm$.007$\ $ & .140$\pm$.006$\bullet$ & .086$\pm$.004$\circ$ & N/A$\ $ & .131$\pm$.007$\bullet$ & .084$\pm$.003$\circ$ & .110$\pm$.004$\bullet$\\
\textsf{o.e.} $^{\downarrow}$ & .295$\pm$.025$\ $ & .415$\pm$.026$\bullet$ & .341$\pm$.031$\bullet$ & N/A$\ $ & .440$\pm$.031$\bullet$ & .320$\pm$.029$\bullet$ & .302$\pm$.028$\ $\\
\textsf{co.} $^{\downarrow}$ & .238$\pm$.014$\ $ & .837$\pm$.018$\bullet$ & .254$\pm$.015$\bullet$ & N/A$\ $ & .312$\pm$.020$\bullet$ & .256$\pm$.018$\bullet$ & .239$\pm$.013$\ $\\
\textsf{r.l.} $^{\downarrow}$ & .108$\pm$.009$\ $ & .675$\pm$.017$\bullet$ & .131$\pm$.010$\bullet$ & N/A$\ $ & .165$\pm$.013$\bullet$ & .125$\pm$.011$\bullet$ & .107$\pm$.007$\ $\\
\textsf{a.p.} $^{\uparrow}$ & .688$\pm$.017$\ $ & .326$\pm$.016$\bullet$ & .666$\pm$.018$\bullet$ & N/A$\ $ & .591$\pm$.018$\bullet$ & .685$\pm$.018$\ $ & .687$\pm$.013$\ $\\

\hline
\multicolumn{5}{l}{\emph{Reuters}}\raisebox{0.5em}{$\ $}\\
\hline
\textsf{h.l.} $^{\downarrow}$ & .028$\pm$.004$\ $ & .043$\pm$.004$\bullet$ & .032$\pm$.003$\bullet$ & N/A$\ $ & .034$\pm$.004$\bullet$ & .042$\pm$.004$\bullet$ & .037$\pm$.003$\bullet$\\
\textsf{o.e.} $^{\downarrow}$ & .044$\pm$.008$\ $ & .077$\pm$.011$\bullet$ & .057$\pm$.010$\bullet$ & N/A$\ $ & .065$\pm$.011$\bullet$ & .100$\pm$.015$\bullet$ & .055$\pm$.007$\bullet$\\
\textsf{co.} $^{\downarrow}$ & .035$\pm$.004$\ $ & .089$\pm$.010$\bullet$ & .036$\pm$.004$\bullet$ & N/A$\ $ & .043$\pm$.004$\bullet$ & .050$\pm$.006$\bullet$ & .036$\pm$.004$\bullet$\\
\textsf{r.l.} $^{\downarrow}$ & .014$\pm$.004$\ $ & .062$\pm$.008$\bullet$ & .016$\pm$.003$\bullet$ & N/A$\ $ & .023$\pm$.004$\bullet$ & .031$\pm$.005$\bullet$ & .016$\pm$.003$\bullet$\\
\textsf{a.p.} $^{\uparrow}$ & .972$\pm$.005$\ $ & .922$\pm$.008$\bullet$ & .966$\pm$.006$\bullet$ & N/A$\ $ & .958$\pm$.006$\bullet$ & .939$\pm$.009$\bullet$ & .967$\pm$.005$\bullet$\\

\hline
\multicolumn{5}{l}{\emph{Bird Song}}\raisebox{0.5em}{$\ $}\\
\hline
\textsf{h.l.} $^{\downarrow}$ & .073$\pm$.009$\ $ & .116$\pm$.005$\bullet$ & .098$\pm$.011$\bullet$ & N/A$\ $ & .081$\pm$.007$\bullet$ & .073$\pm$.005$\ $ & .087$\pm$.008$\bullet$\\
\textsf{o.e.} $^{\downarrow}$ & .055$\pm$.017$\ $ & .101$\pm$.020$\bullet$ & .159$\pm$.039$\bullet$ & N/A$\ $ & .122$\pm$.029$\bullet$ & .111$\pm$.025$\bullet$ & .064$\pm$.046$\ $\\
\textsf{co.} $^{\downarrow}$ & .150$\pm$.013$\ $ & .292$\pm$.015$\bullet$ & .186$\pm$.018$\bullet$ & N/A$\ $ & .175$\pm$.015$\bullet$ & .173$\pm$.013$\bullet$ & .133$\pm$.011$\circ$\\
\textsf{r.l.} $^{\downarrow}$ & .036$\pm$.007$\ $ & .132$\pm$.010$\bullet$ & .067$\pm$.012$\bullet$ & N/A$\ $ & .059$\pm$.010$\bullet$ & .054$\pm$.006$\bullet$ & .027$\pm$.008$\circ$\\
\textsf{a.p.} $^{\uparrow}$ & .921$\pm$.014$\ $ & .786$\pm$.013$\bullet$ & .847$\pm$.026$\bullet$ & N/A$\ $ & .878$\pm$.017$\bullet$ & .888$\pm$.011$\bullet$ & .930$\pm$.025$\ $\\

\hline
\multicolumn{5}{l}{\emph{Scene}}\raisebox{0.5em}{$\ $}\\
\hline
\textsf{h.l.} $^{\downarrow}$ & .188$\pm$.009$\ $ & .269$\pm$.009$\bullet$ & .194$\pm$.005$\bullet$ & N/A$\ $ & .196$\pm$.007$\bullet$ & .200$\pm$.008$\bullet$ & .204$\pm$.007$\bullet$\\
\textsf{o.e.} $^{\downarrow}$ & .351$\pm$.023$\ $ & .386$\pm$.025$\bullet$ & .351$\pm$.020$\ $ & N/A$\ $ & .370$\pm$.018$\bullet$ & .380$\pm$.021$\bullet$ & .392$\pm$.019$\bullet$\\
\textsf{co.} $^{\downarrow}$ & .207$\pm$.012$\ $ & .334$\pm$.011$\bullet$ & .204$\pm$.008$\circ$ & N/A$\ $ & .222$\pm$.009$\bullet$ & .225$\pm$.010$\bullet$ & .237$\pm$.010$\bullet$\\
\textsf{r.l.} $^{\downarrow}$ & .189$\pm$.014$\ $ & .348$\pm$.012$\bullet$ & .185$\pm$.010$\ $ & N/A$\ $ & .207$\pm$.011$\bullet$ & .212$\pm$.011$\bullet$ & .222$\pm$.010$\bullet$\\
\textsf{a.p.} $^{\uparrow}$ & .770$\pm$.015$\ $ & .600$\pm$.013$\bullet$ & .772$\pm$.012$\ $ & N/A$\ $ & .757$\pm$.011$\bullet$ & .750$\pm$.012$\bullet$ & .738$\pm$.011$\bullet$\\
\hline
\end{tabular}
\end{center}
\end{table}

We perform the experiments on 6 moderate-sized data sets and 2 large data sets. Among the moderate-sized data sets, \emph{Scene} and \emph{Reuters} are two benchmark data sets commonly used in existing MIML works. \emph{Scene} \cite{ZZ07} consists of 2000 images for scene classification, and is associated with 5 possible labels: desert, mountains, sea, sunset and trees. For each image, a bag of 9 instances is extracted via SBN \cite{MR98}. \emph{Reuters} is constructed based on the Reuters-21578 data set \cite{S02} with the sliding window technique in \cite{ATH02}. The other four moderate-sized data sets are collected by Fern et al. in their recent work \cite{BFR12}: \emph{Letter Carroll} and \emph{Letter Frost} are constructed using the UCI Letter Recognition dataset \cite{FS91}, where a bag is created for each word, and labels correspond to the letters. \emph{Bird Song} consists of bird song recordings at the H. J. Andrews (HJA) Experimental Forest. Each bag is extracted from a 10-seconds audio recording while labels correspond to species of birds. \emph{MSRC v2} is a subset of the Microsoft Research Cambridge (MSRC) image dataset \cite{WCM05}. Based on the ground-truth segmentation, histograms of gradients and colors are extracted to form an instance for each segment. The two large data sets are \emph{Corel5K} and \emph{MSRA}. \emph{Corel5K} \cite{DBDF02} contains 5000 segmented images and 260 class labels, and each image is represented by 9 instances on average. \emph{MSRA} \cite{LWH09} is a multimedia database collected by Microsoft Research Asia, the subset used in this work contains 30000 images with 99 possible labels, and each image is represented with a bag of 9 instances.  The detailed characteristics of these data sets are summarized in Table \ref{table:data}.

For \emph{MSRA} and \emph{Corel5K}, since existing MIML approaches cannot handle large scale data, we examine the performances of compared approaches on a series of subsets with different number of training bags (which will be specified later). For each data set, $2/3$ of the data are randomly sampled for training, and the remaining examples are taken as test set. We repeat the random data partition for thirty times, and report the average results over the thirty repetitions.

For MIMLfast, the step size is in the form $\gamma_t=\gamma_0/(1+\eta \gamma_0 t)$ according to \cite{B10}. The parameters are selected by 3-fold cross validation on the training data with regard to ranking loss. The candidate values for the parameters are as below: $m\in\{50,100,200\}$, $C\in\{1, 5, 10\}$, $K\in\{1,5,10,15\}$, $\gamma_0\in\{0.0001,0.0005,0.001,0.005\}$ and $\eta\in\{10^{-5},10^{-6}\}$. In our experience, the algorithm is not very sensitive to $m$ and $C$; and the influence of $K$ will be studied in Section 3.5. For the compared approaches, parameters are determined in the same way if no value suggested in their literatures.

\begin{figure*}[!t]
\begin{center}
$\qquad$\begin{minipage}{0.22\linewidth}
\includegraphics[width=\textwidth]{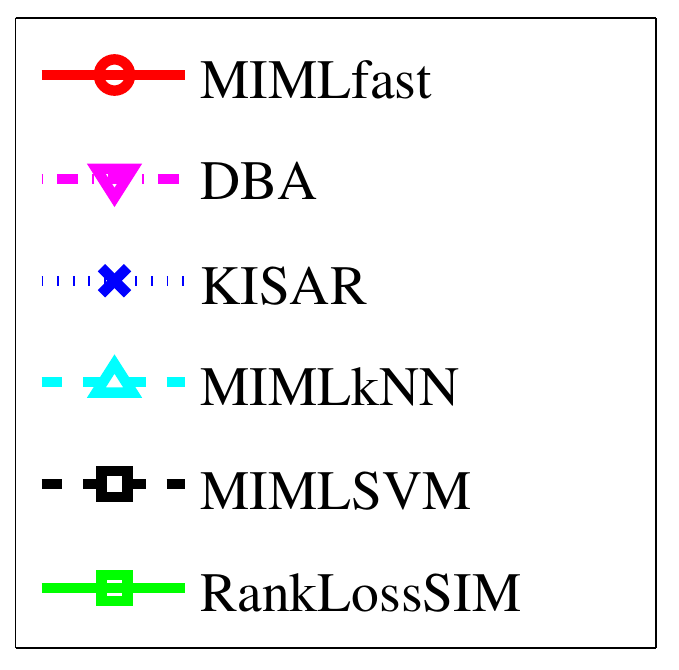}\\
\centering{(a) \textsf{Legend}}
\end{minipage}$\ \ \quad$
\begin{minipage}{0.32\linewidth}
\includegraphics[width=\textwidth]{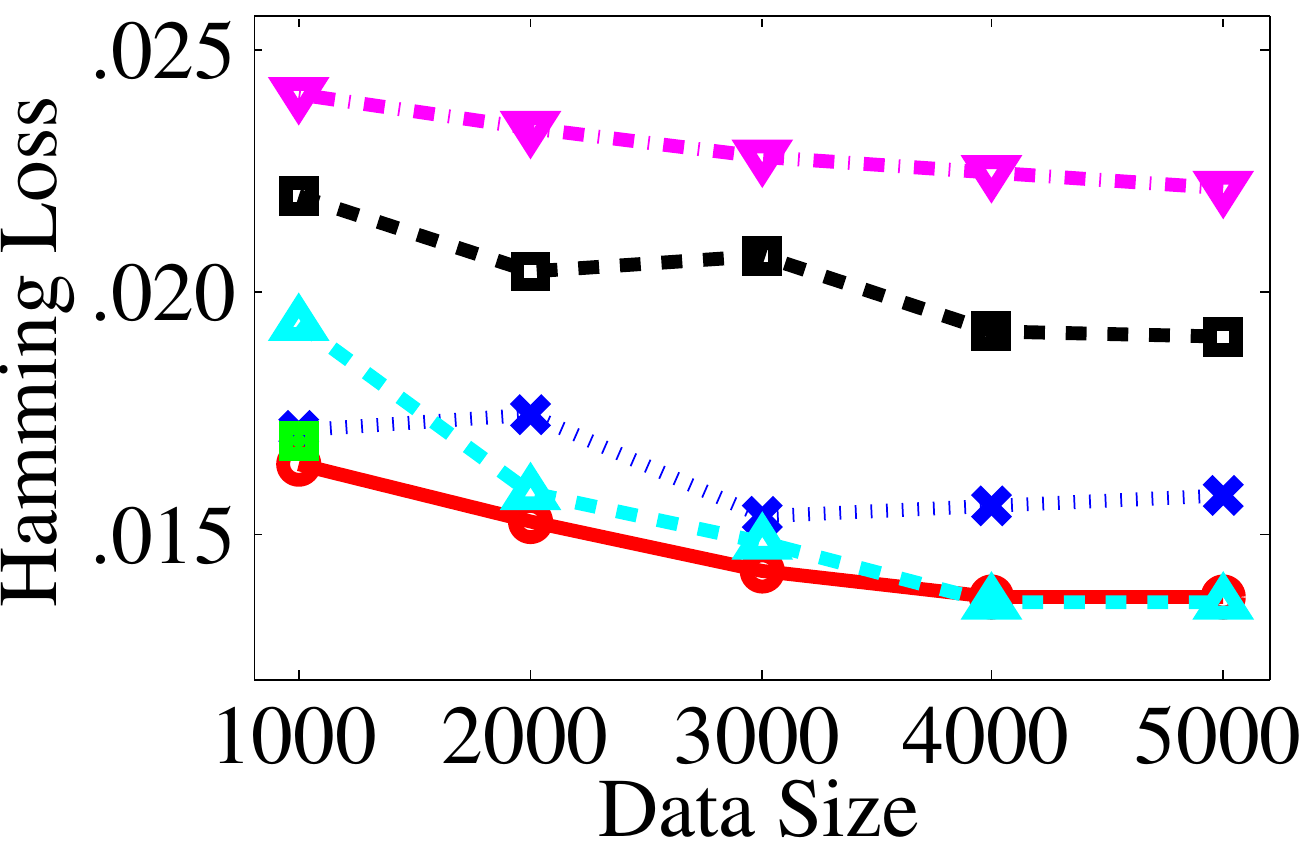}\\
\centering{(b) \textsf{Hamming Loss} $\downarrow$}
\end{minipage}
\begin{minipage}{0.32\linewidth}
\includegraphics[width=\textwidth]{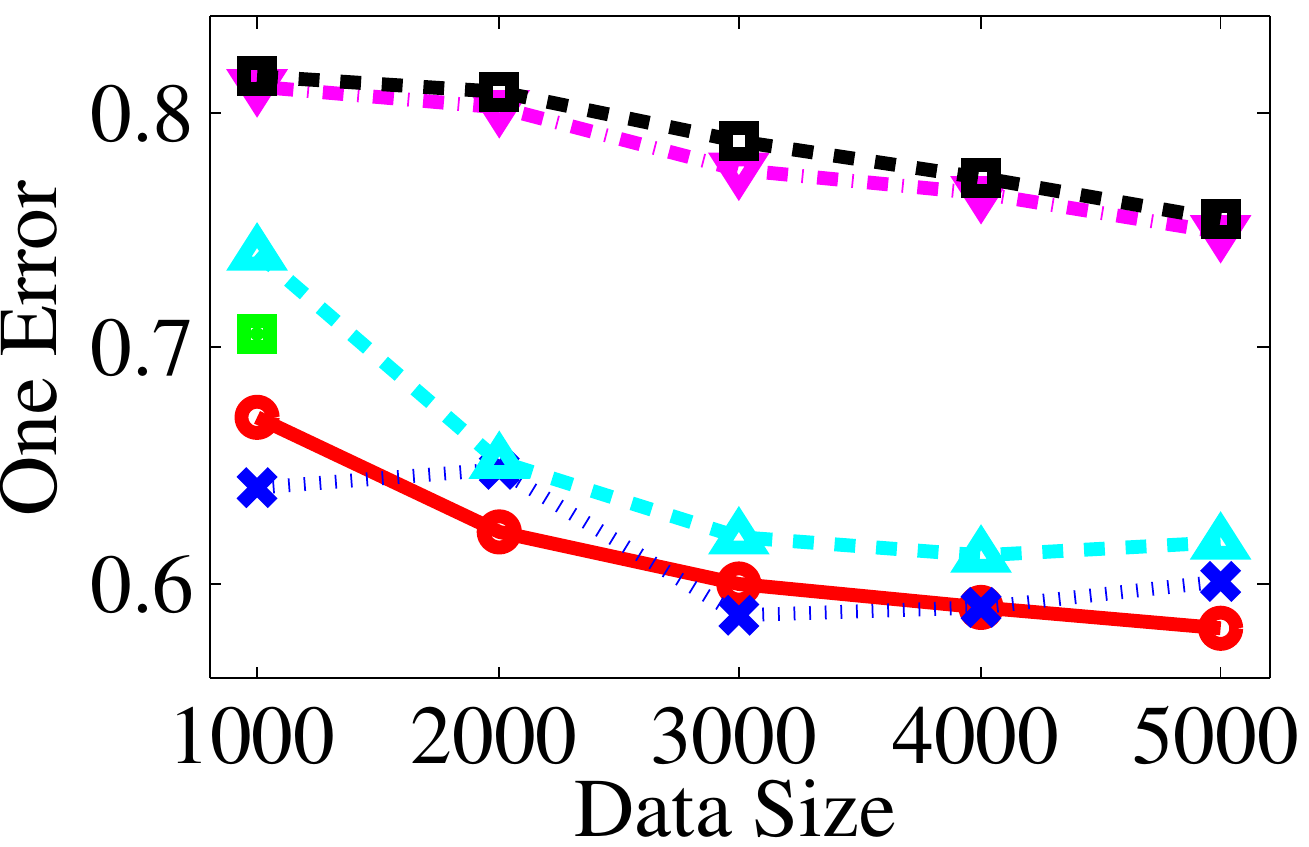}\\
\centering{(c) \textsf{One Error} $\downarrow$}
\end{minipage}\\
\begin{minipage}{0.32\linewidth}
\includegraphics[width=\textwidth]{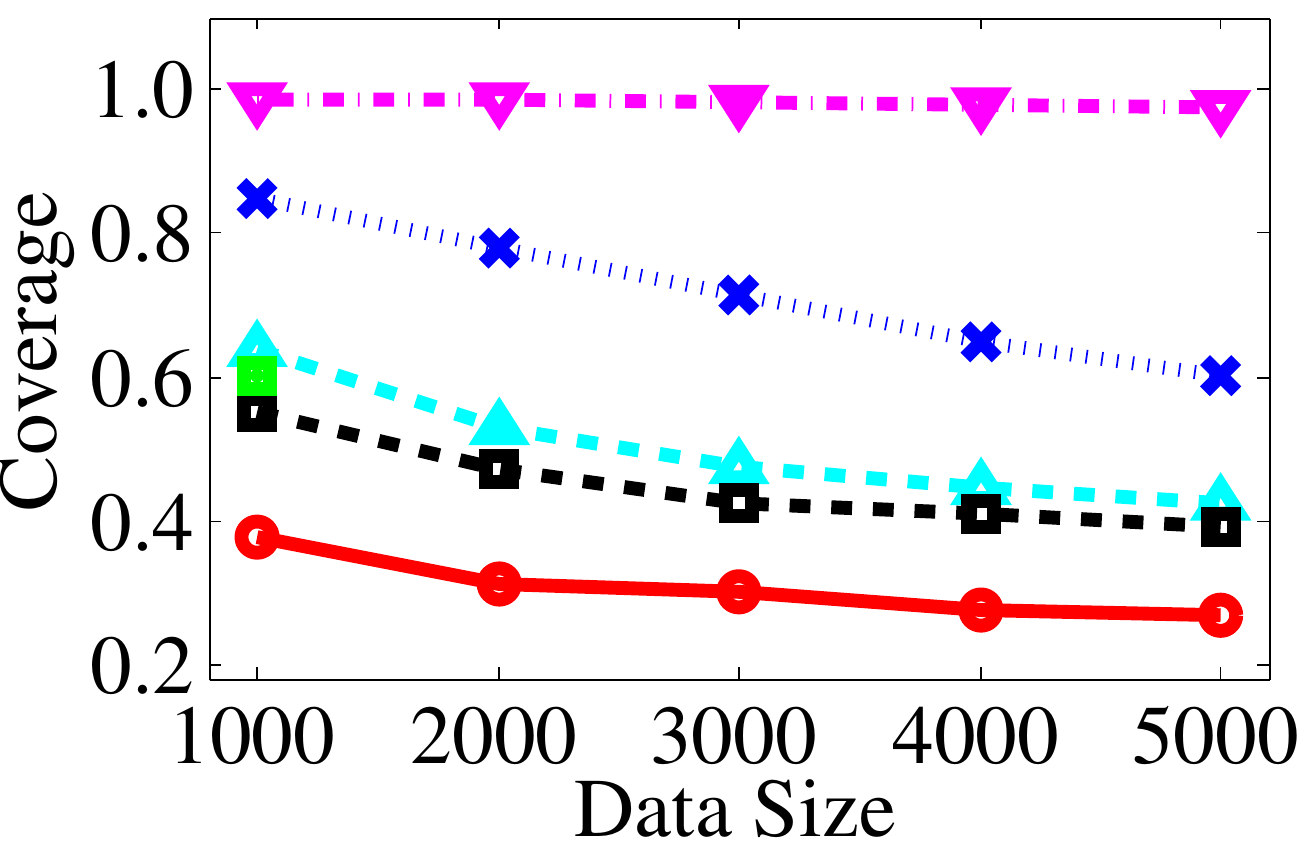}\\
\centering{(d) \textsf{Coverage} $\downarrow$}
\end{minipage}
\begin{minipage}{0.32\linewidth}
\includegraphics[width=\textwidth]{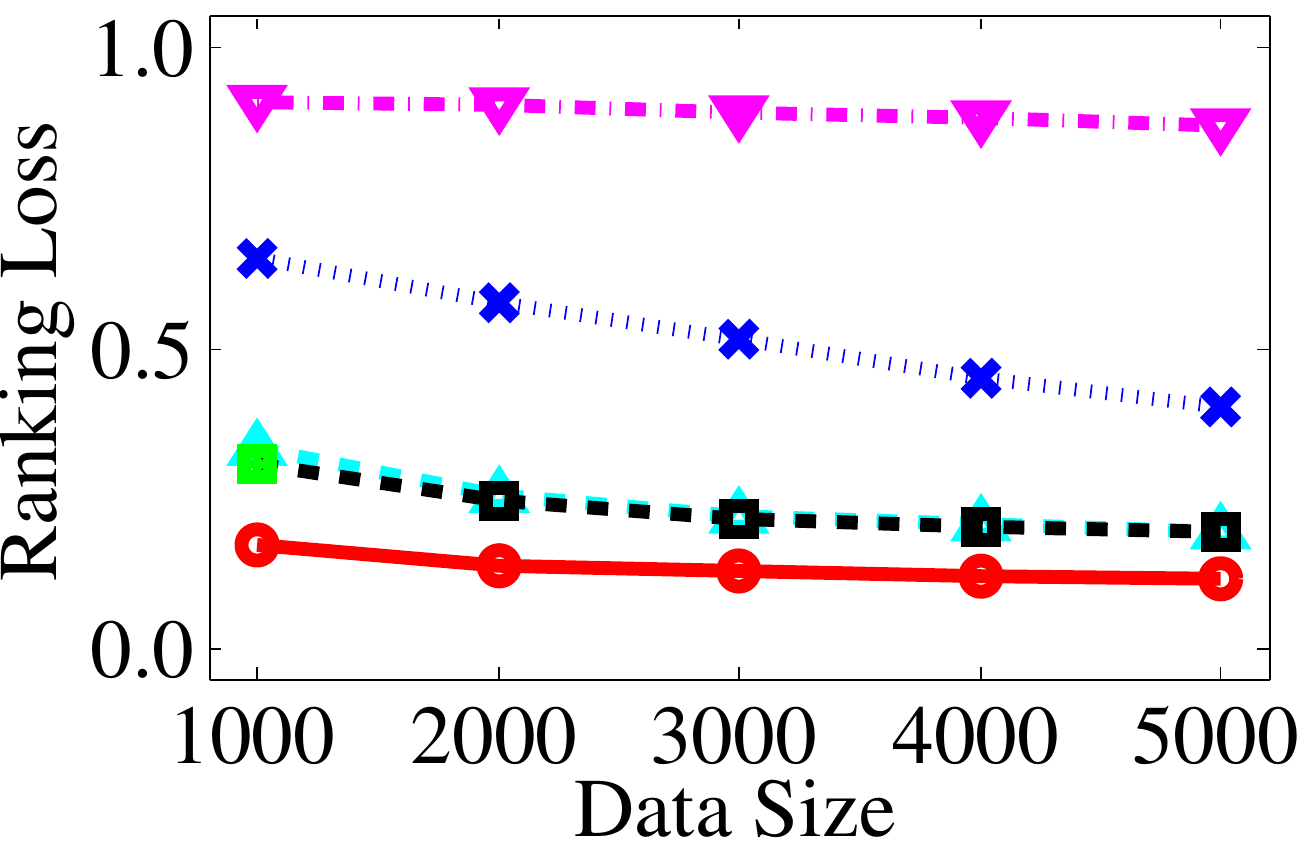}\\
\centering{(e) \textsf{Ranking Loss} $\downarrow$}
\end{minipage}
\begin{minipage}{0.32\linewidth}
\includegraphics[width=\textwidth]{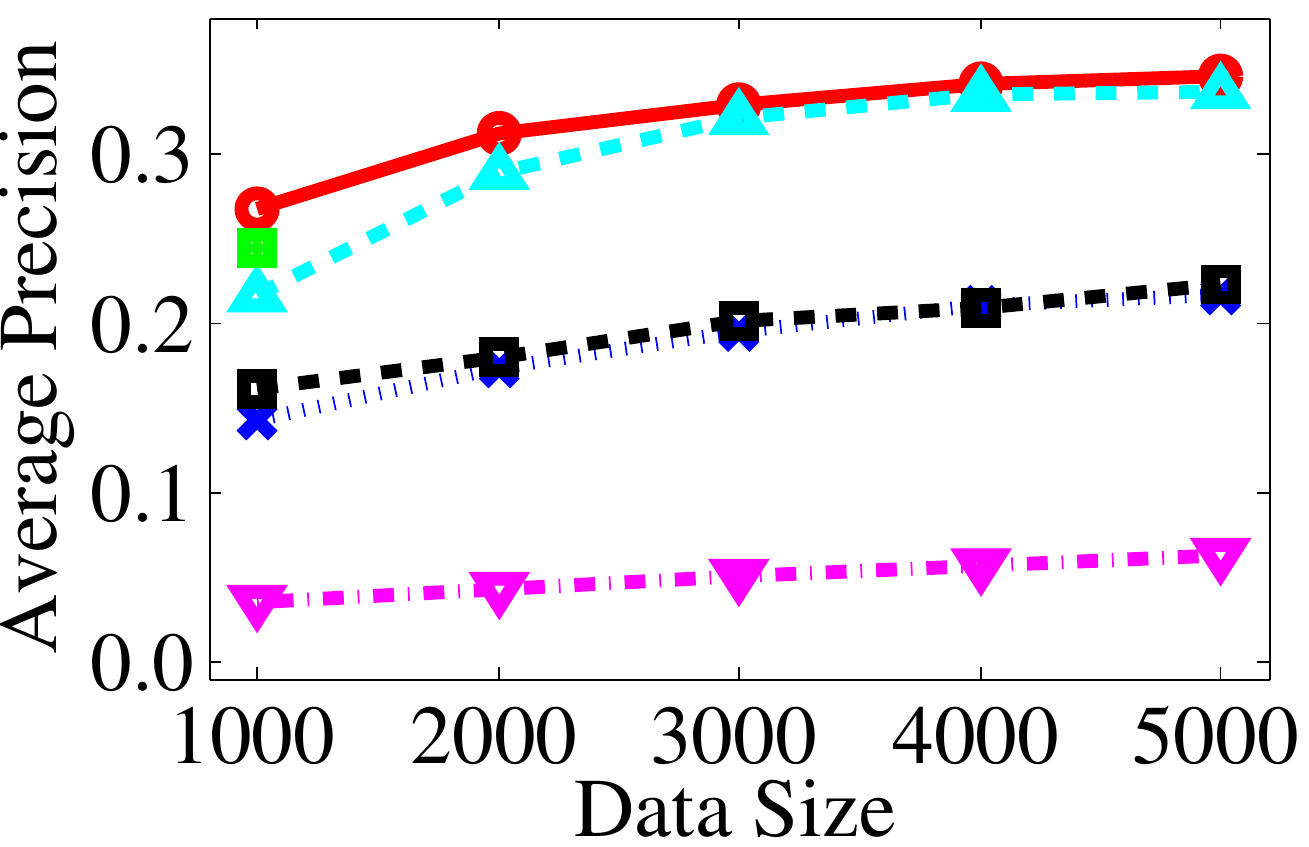}\\
\centering{(f) \textsf{Average Precision} $\uparrow$}
\end{minipage}
\caption{Comparison results on \emph{Corel5K} with varying data size; $\uparrow$($\downarrow$) indicates that the larger (smaller) the value, the better the performance.}\label{fig:corel}
\end{center}
\vspace{-0.2cm}
\end{figure*}

\begin{figure*}[!t]
\begin{center}
$\quad$\begin{minipage}{0.28\linewidth}
\includegraphics[width=\textwidth]{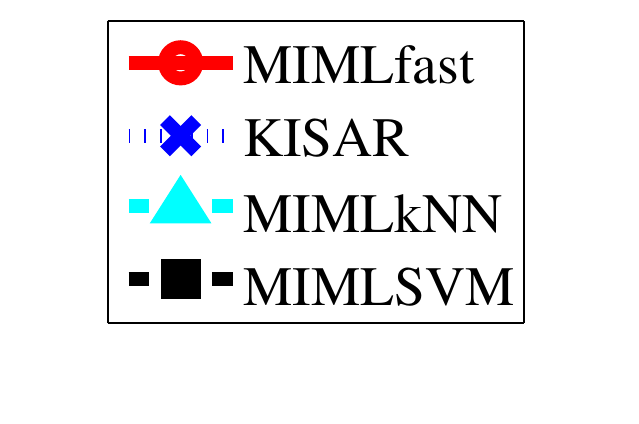}\\
\centering{(a) \textsf{Legend}}
\end{minipage}$\quad$
\begin{minipage}{0.32\linewidth}
\includegraphics[width=\textwidth]{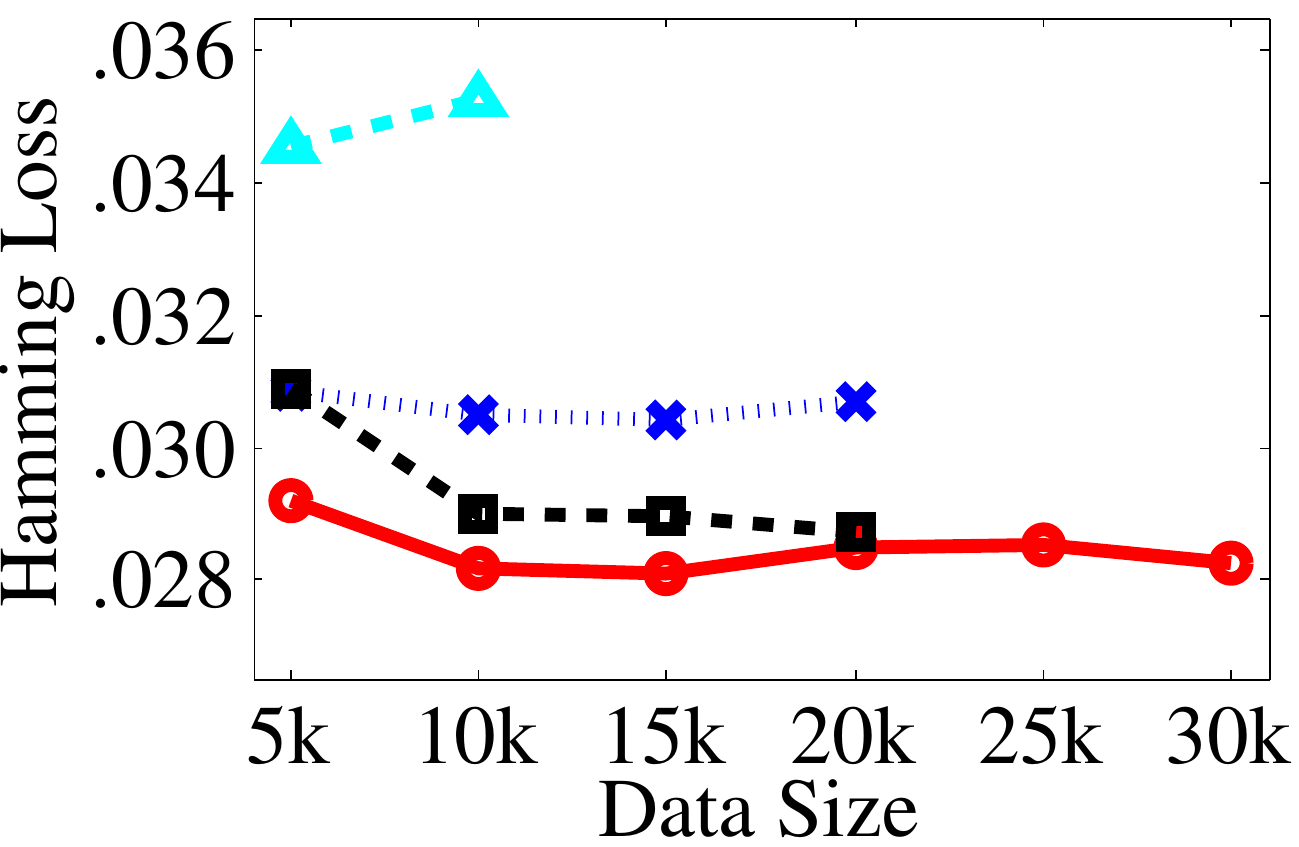}\\
\centering{(b) \textsf{Hamming Loss} $\downarrow$}
\end{minipage}
\begin{minipage}{0.32\linewidth}
\includegraphics[width=\textwidth]{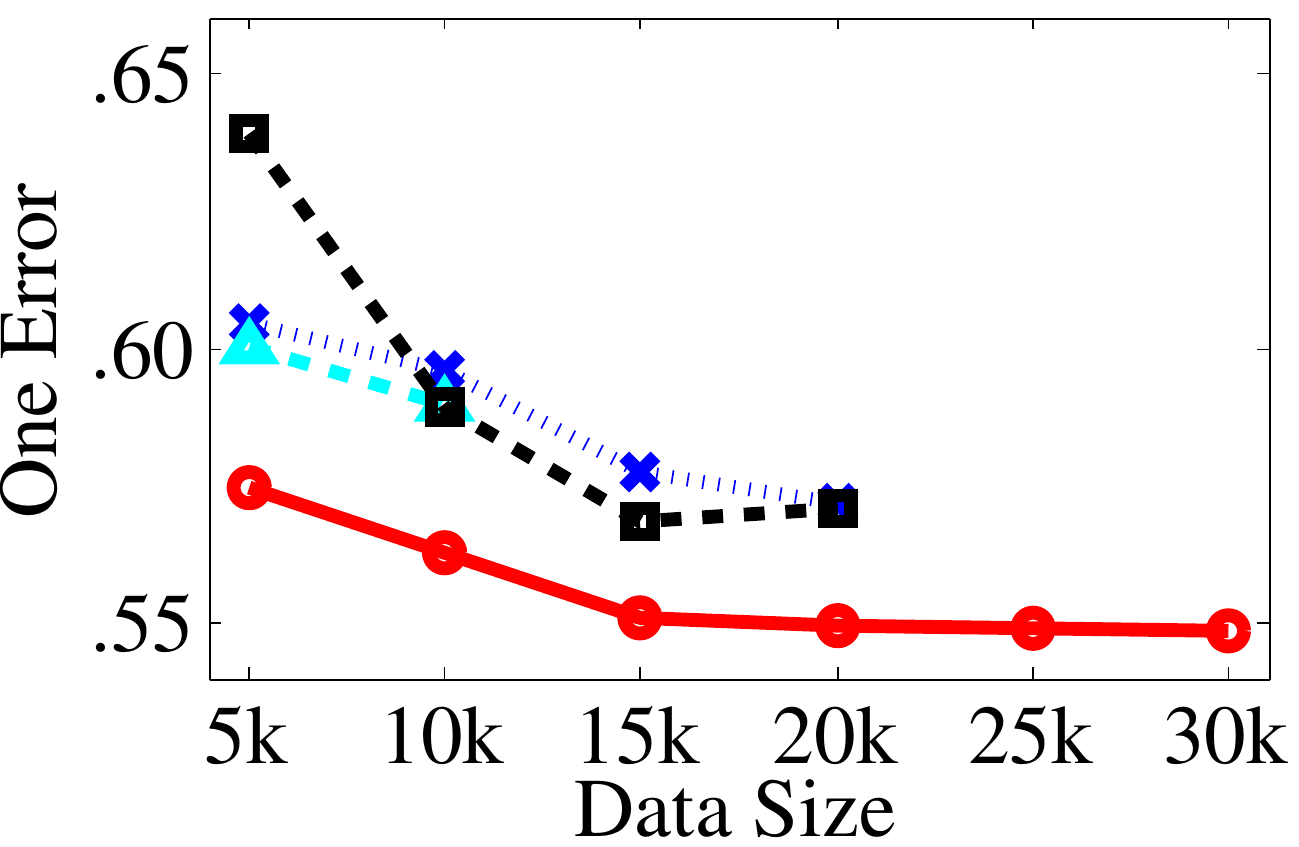}\\
\centering{(c) \textsf{One Error} $\downarrow$}
\end{minipage}\\
\begin{minipage}{0.32\linewidth}
\includegraphics[width=\textwidth]{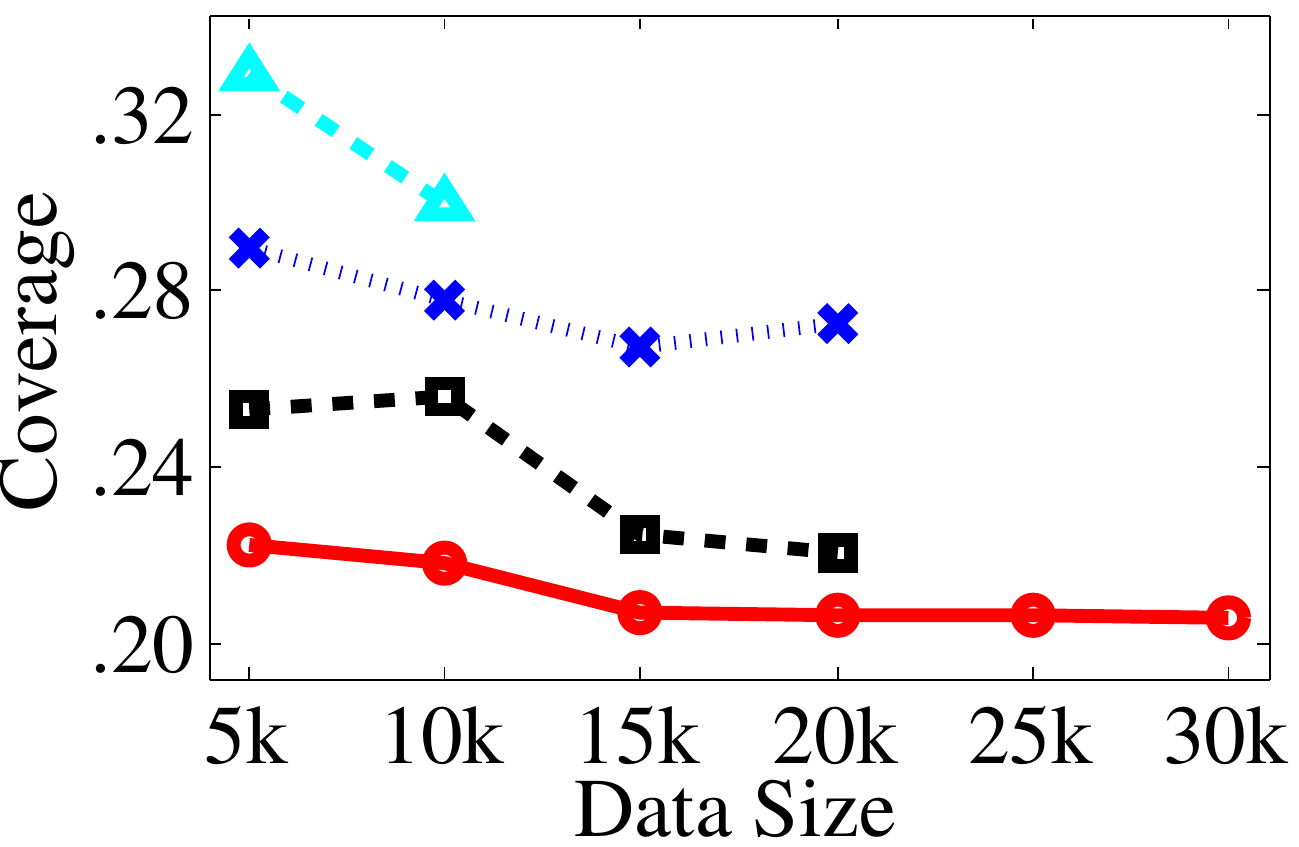}\\
\centering{(d) \textsf{Coverage} $\downarrow$}
\end{minipage}
\begin{minipage}{0.32\linewidth}
\includegraphics[width=\textwidth]{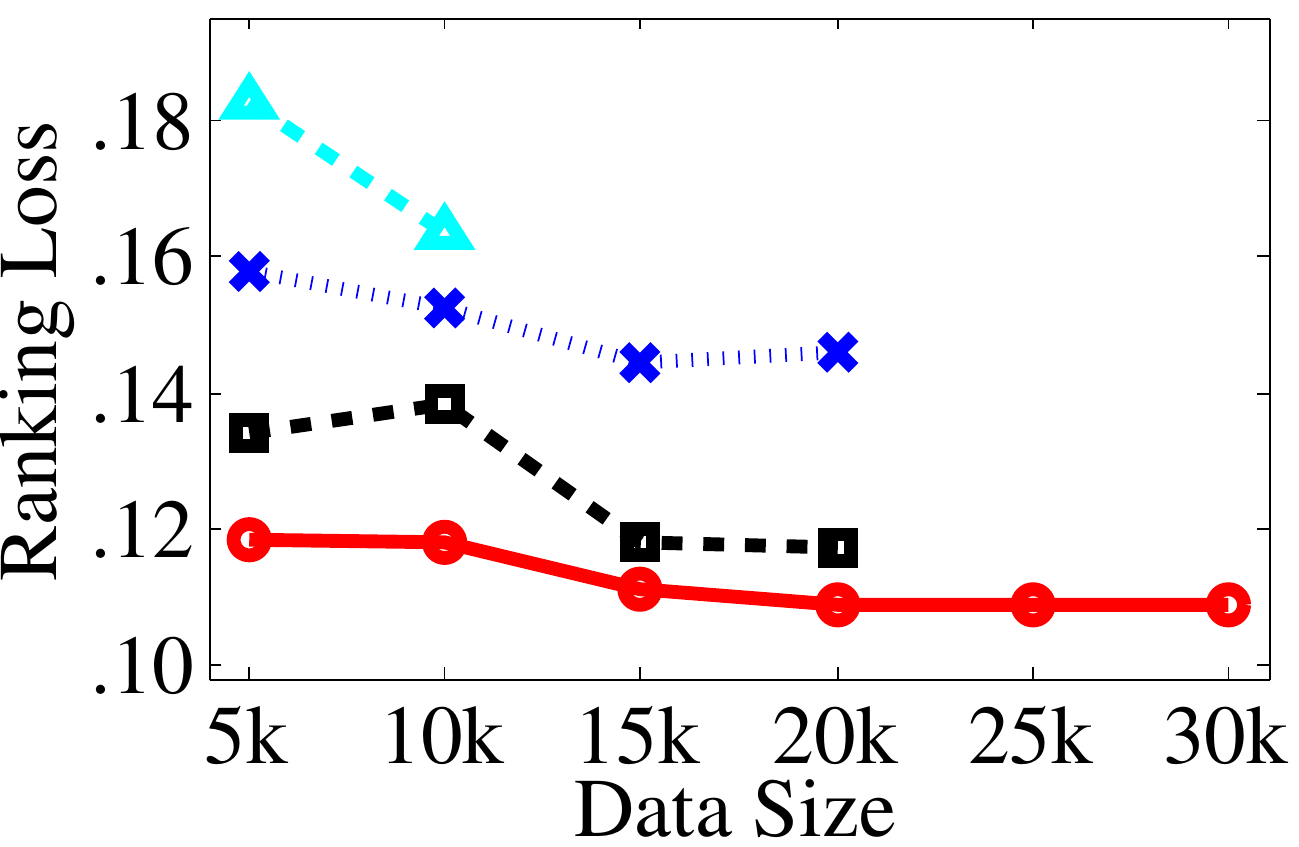}\\
\centering{(e) \textsf{Ranking Loss} $\downarrow$}
\end{minipage}
\begin{minipage}{0.32\linewidth}
\includegraphics[width=\textwidth]{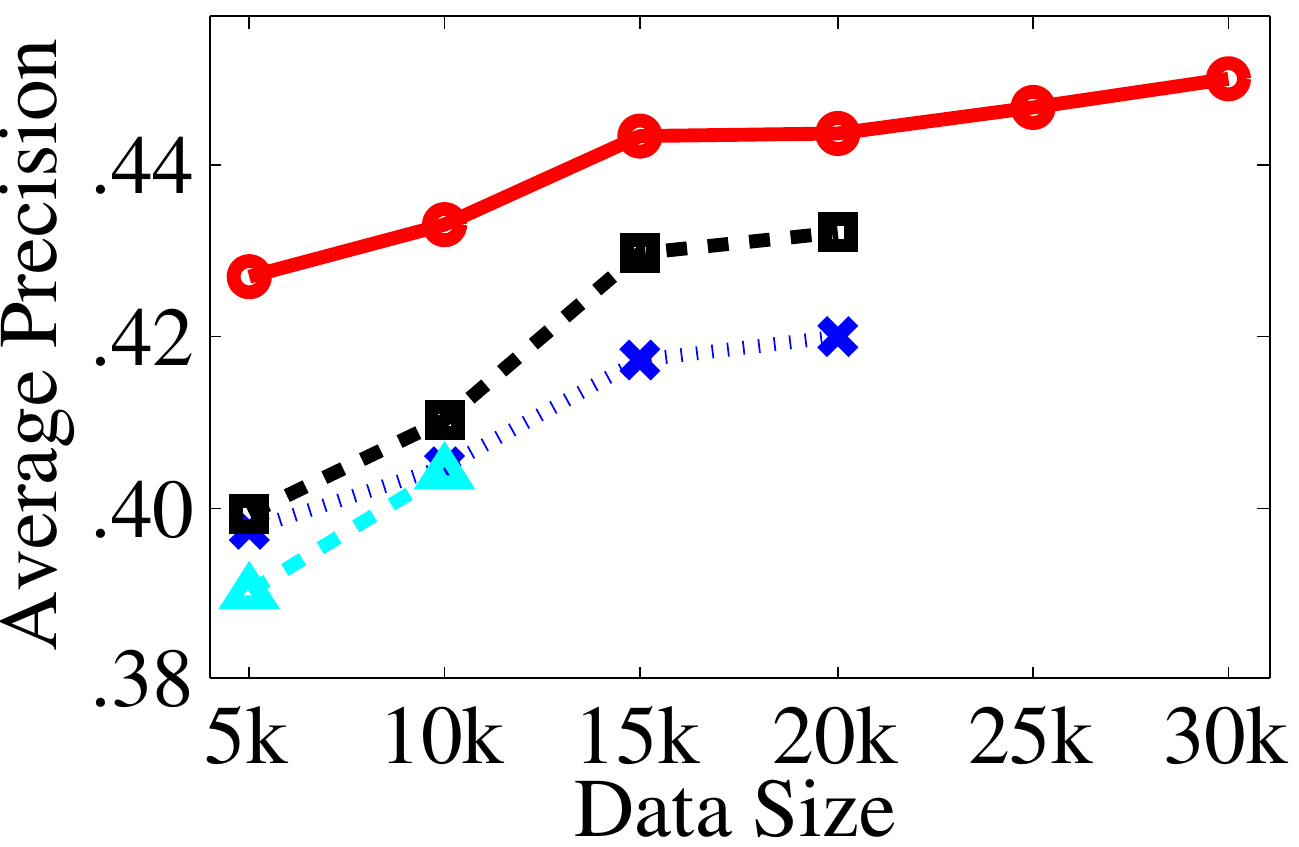}\\
\centering{(f) \textsf{Average Precision} $\uparrow$}
\end{minipage}
\caption{Comparison results on \emph{MSRA} with varying data size; $\uparrow$($\downarrow$) indicates that the larger (smaller) the value, the better the performance; only MIMLfast can work when data size reaches 25,000.}\label{fig:msra}
\end{center}
\vspace{-0.2cm}
\end{figure*}

The performances of the compared approaches are evaluated with five commonly used MIML criteria: \textsf{hamming loss}, \textsf{one error}, \textsf{coverage}, \textsf{ranking loss} and \textsf{average precision}. For \textsf{average precision}, a larger value implies a better performance, while for the other four criteria, the smaller, the better. Note that \textsf{coverage} is normalized by the number of labels such that all criteria are in the interval $[0,1]$. The definition of these criteria can be found in \cite{SS00,ZZHL12}.

\subsection{Performance Comparison}
We first report the comparison results on the six moderate-sized data sets in Table \ref{table:result}. As shown in the table, our approach MIMLfast achieves the best performance in most cases. DBA tends to favor text data, and is outperformed by MIMLfast on all the data sets. KISAR achieves comparable results with MIMLfast on \emph{Scene} while is less effective on the other data sets. MIMLBoost can handle only the two smallest data sets, and does not yield good performance. MIMLkNN and MIMLSVM work steady on all the data sets, but are not competitive when compared with MIMLfast. At last, RankLossSIM is comparable to MIMLfast on 4 of 6 data sets, and even achieves better \textsf{coverage} and \textsf{ranking loss} on the \emph{Bird Song} data set. However, on the other two data sets with relative more bags, i.e., \emph{Reuters} and \emph{Scene}, it is significantly worse than our approach on all the five criteria.

\emph{MSRA} and \emph{Corel5K} contain 30000 and 5000 bags respectively, which are too large for most existing MIML approaches. We thus perform the comparison on subsets of them with different data sizes. We vary the number of bags from 1000 to 5000 for \emph{Corel5K}, and 5000 to 30000 for \emph{MSRA}, and plot the performance curves in Figures \ref{fig:corel} and \ref{fig:msra}, respectively. MIMLBoost did not return results in 24 hours even for the smallest data size, and thus it is not included in the comparison. RankLossSIM is not presented on \emph{MSRA} for the same reason. We also exclude DBA on \emph{MSRA} because its performance is too bad. As observable in Figures \ref{fig:corel} and \ref{fig:msra}, MIMLfast is apparently better than the others on these two large data sets. Particularly, when data size reaches 25K, other methods cannot work, but MIMLfast still works well.

\begin{figure*}[!t]
\begin{center}
\begin{minipage}{0.95\linewidth}
\includegraphics[width=\textwidth]{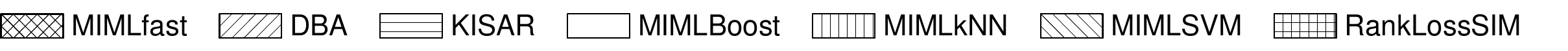}
\end{minipage}\\
\begin{minipage}{0.195\linewidth}
\includegraphics[width=\textwidth]{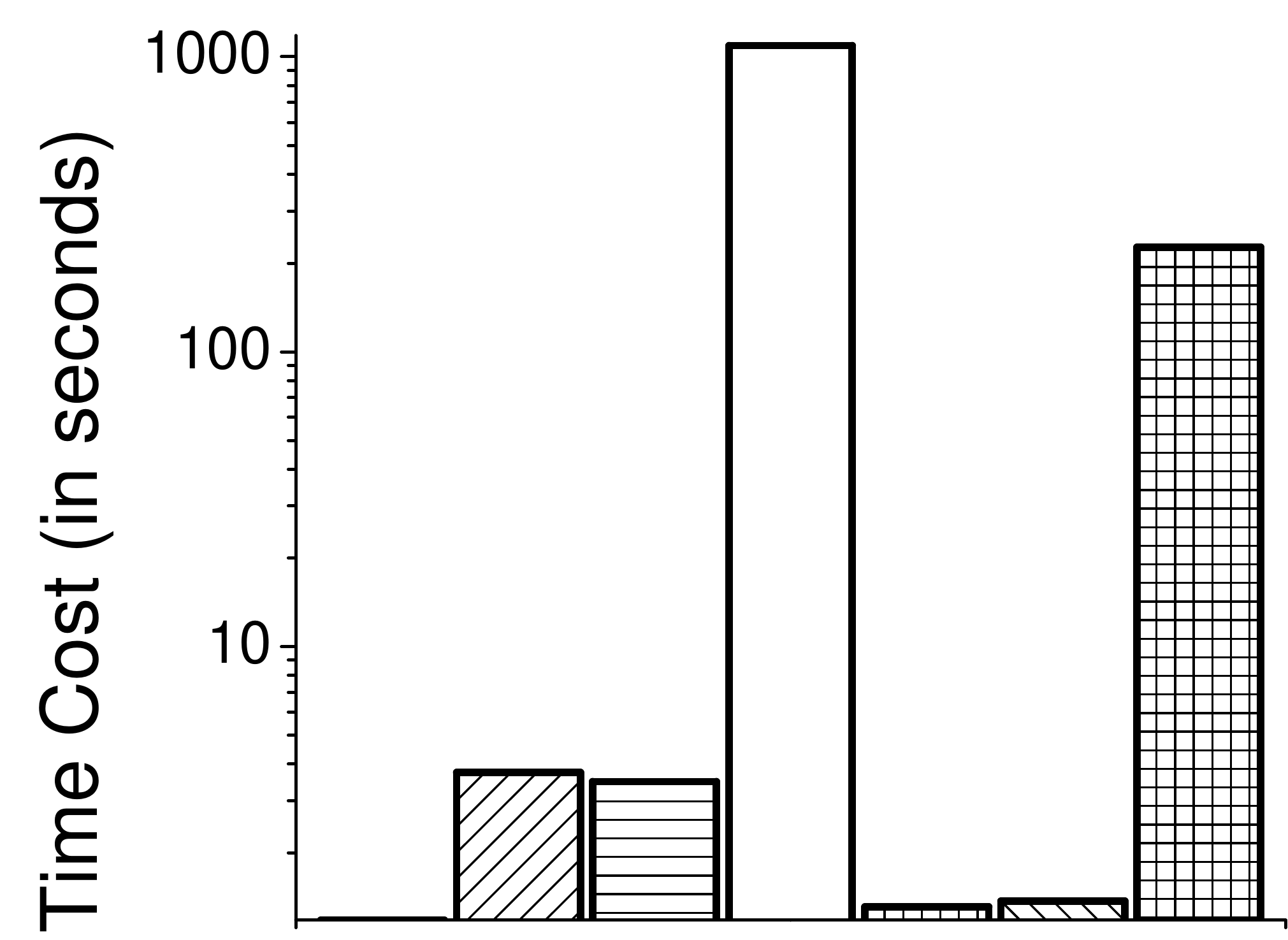}\\
\centering{(a) \emph{Letter Frost}}
\end{minipage}
\begin{minipage}{0.195\linewidth}
\includegraphics[width=\textwidth]{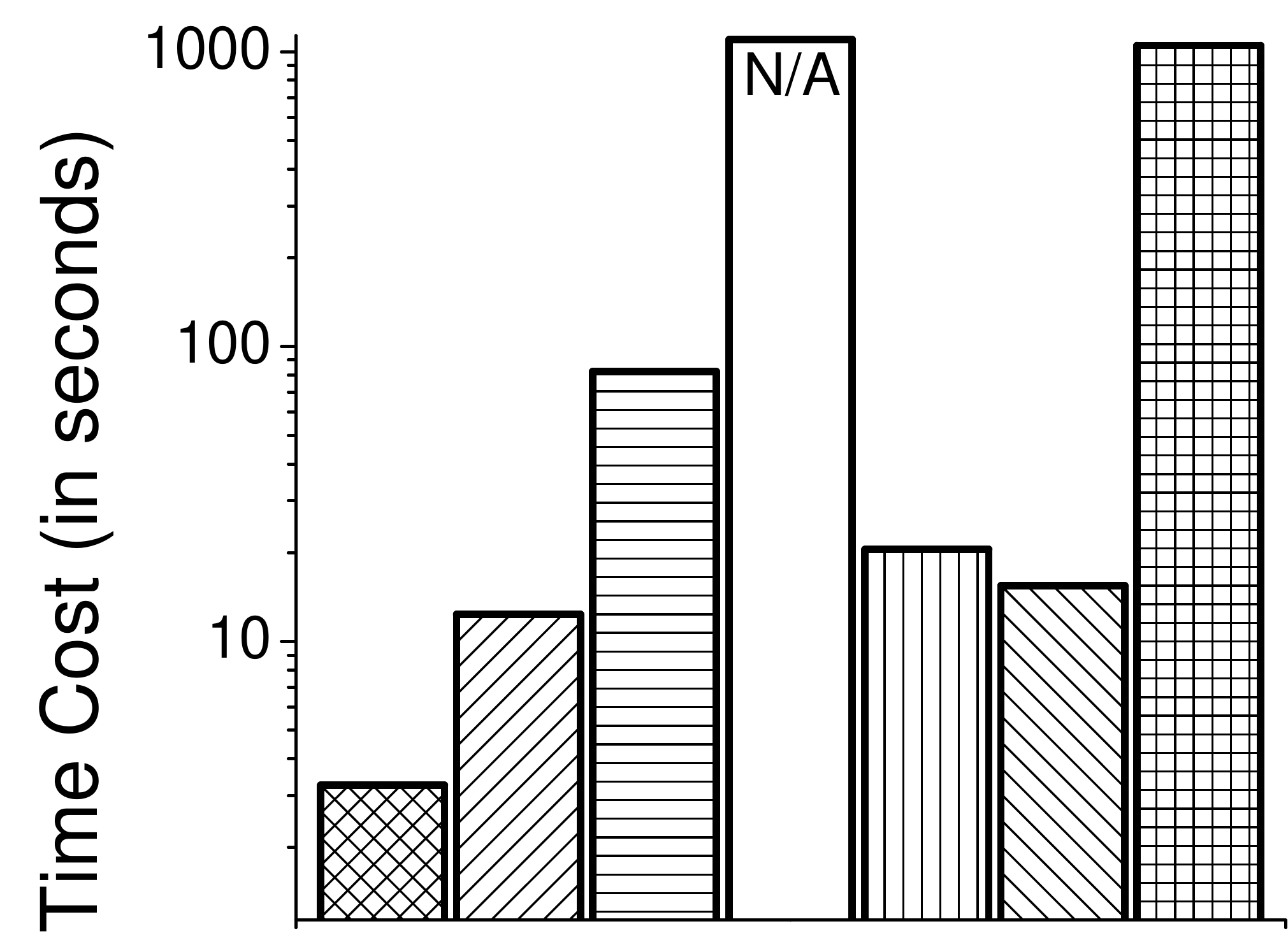}\\
\centering{(b) \emph{MSRC v2}}
\end{minipage}
\begin{minipage}{0.19\linewidth}
\includegraphics[width=\textwidth]{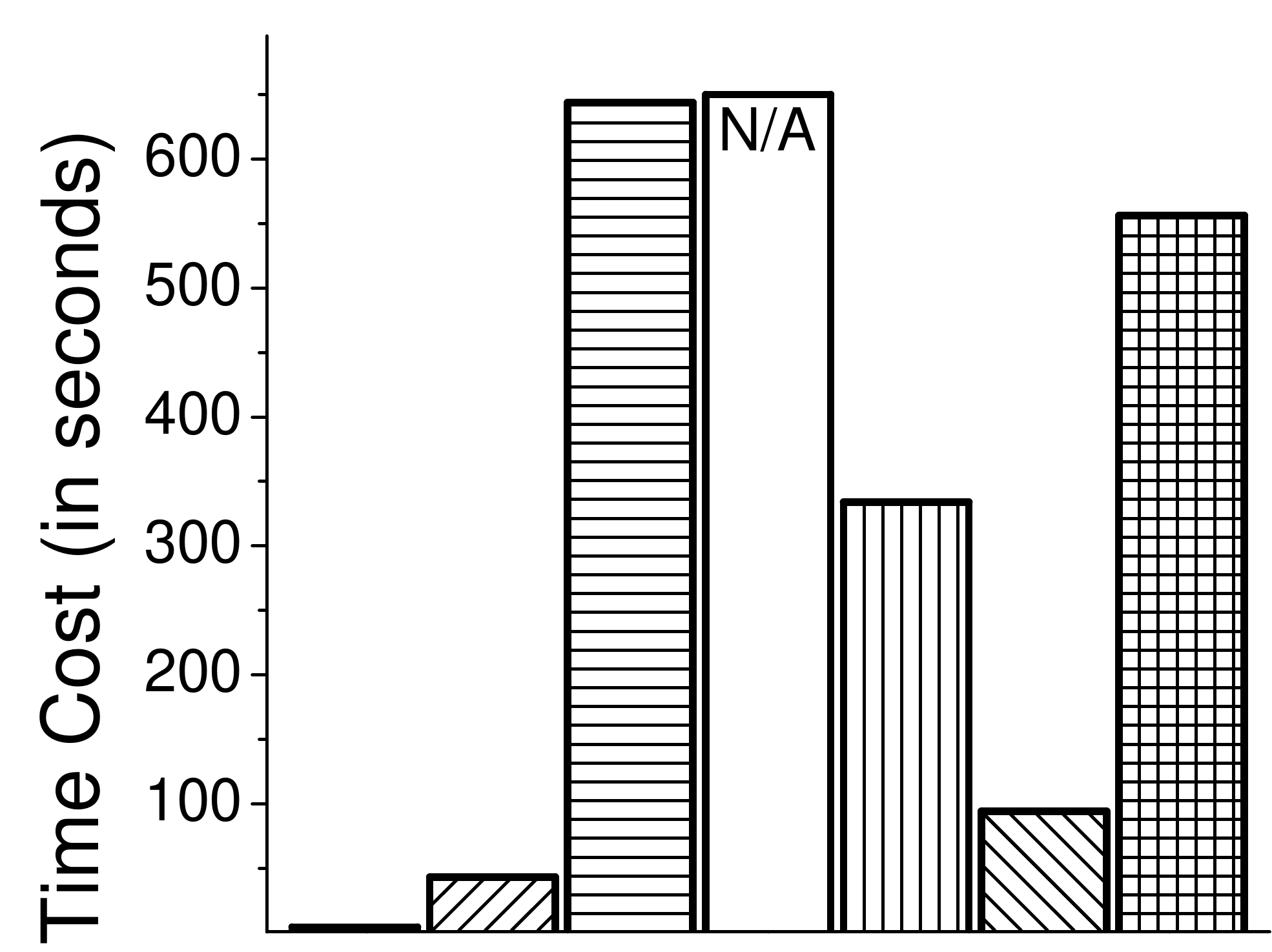}\\
\centering{(c) \emph{Reuters}}
\end{minipage}
\begin{minipage}{0.19\linewidth}
\includegraphics[width=\textwidth]{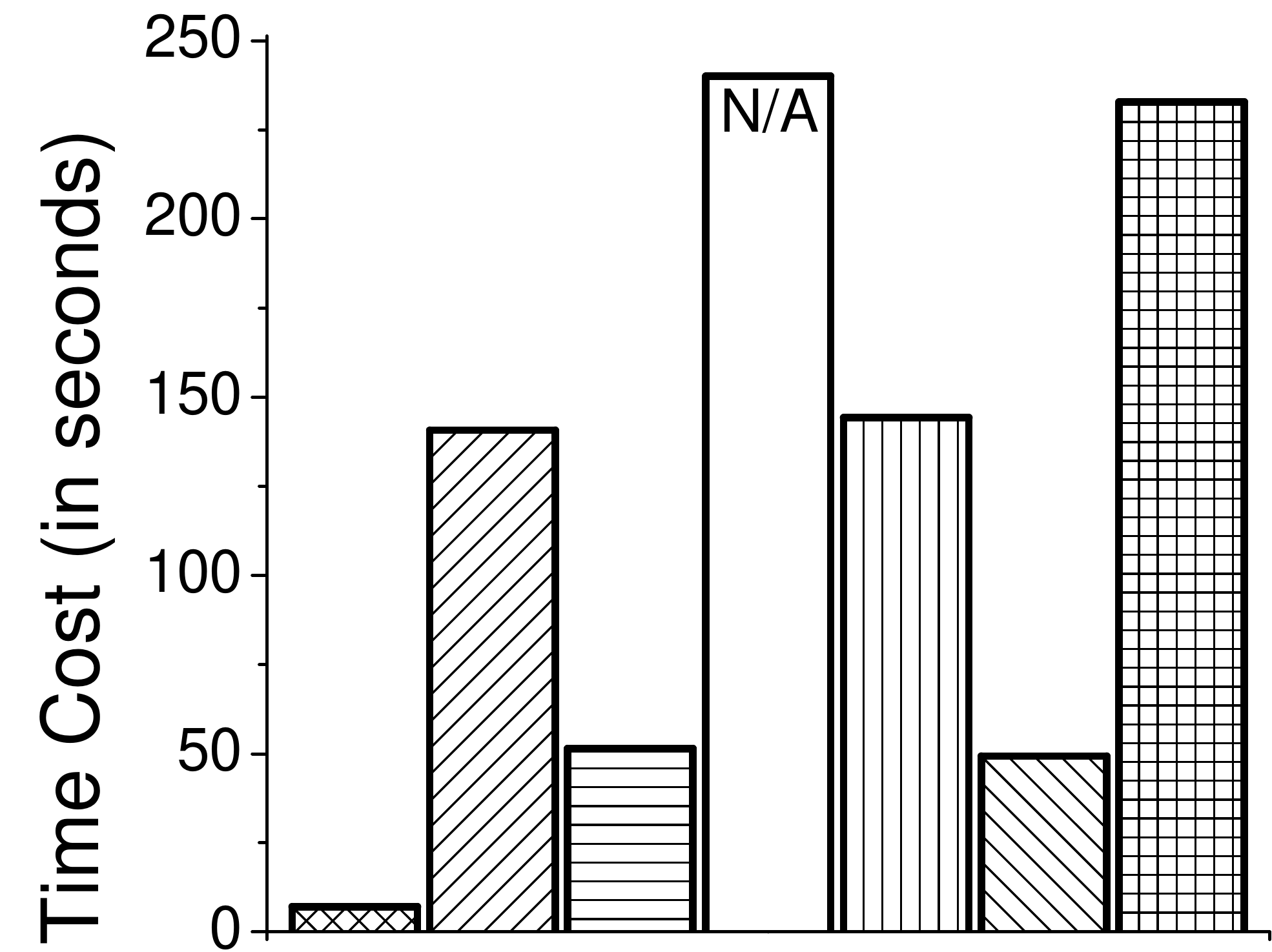}\\
\centering{(d) \emph{Bird Song}}
\end{minipage}
\begin{minipage}{0.19\linewidth}
\includegraphics[width=\textwidth]{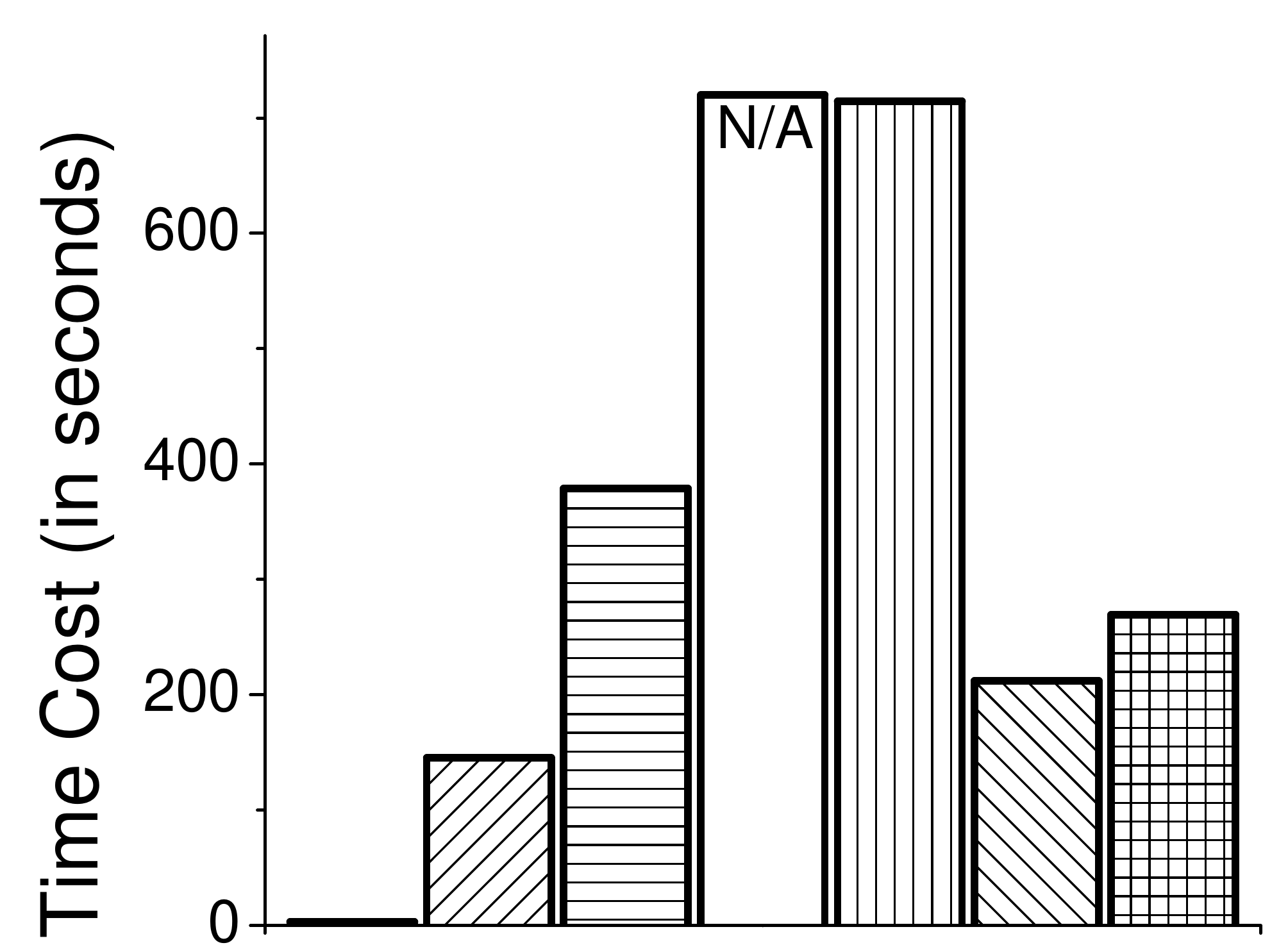}\\
\centering{(e) \emph{Scene}}
\end{minipage}
\caption{Comparison of time cost on six moderate-sized data sets; N/A indicates that no result was obtained in 24 hours; the y-axis in (a) and (b) are log-scaled.}\label{fig:time}
\end{center}
\vspace{-0.2cm}
\end{figure*}

\begin{figure}[!t]
\begin{center}
\begin{minipage}{0.21\linewidth}
\includegraphics[width=\textwidth]{legend2.pdf}
\centering{(a) \emph{legend}}
\end{minipage}
$\quad$
\begin{minipage}{0.32\linewidth}
\includegraphics[width=\textwidth]{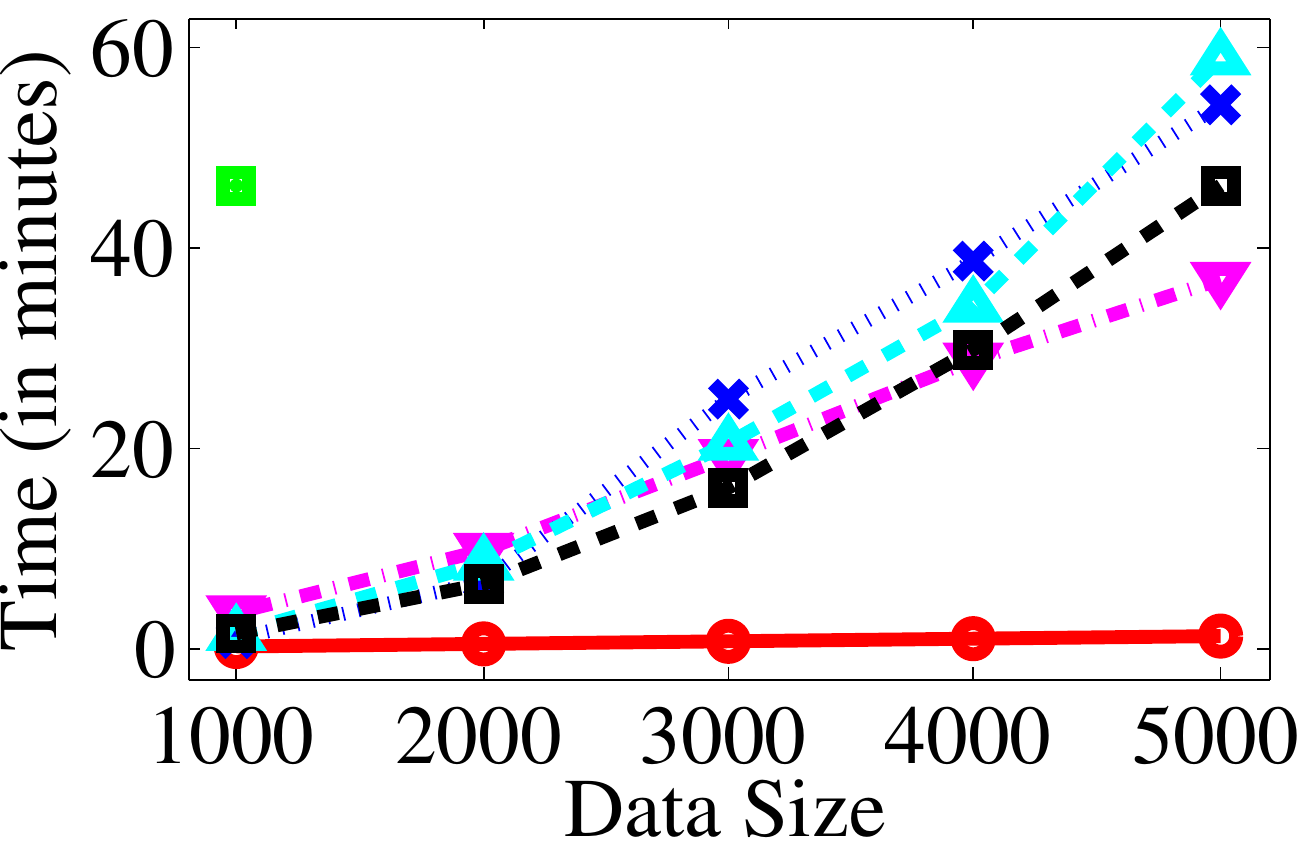}
\centering{(b) \emph{Corel5K}}
\end{minipage}
$\quad$
\begin{minipage}{0.32\linewidth}
\includegraphics[width=\textwidth]{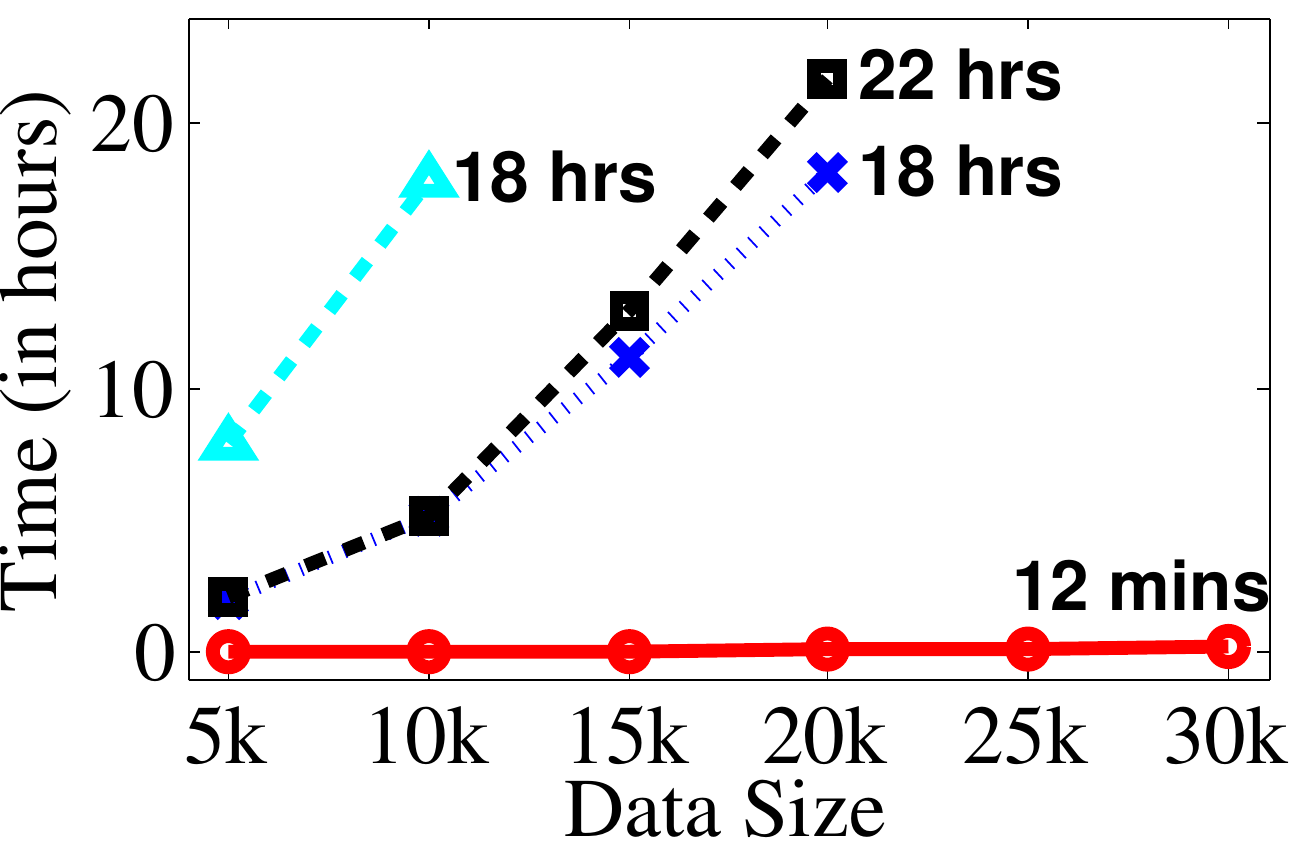}
\centering{(c) \emph{MSRA}}
\end{minipage}
\caption{Comparison of time cost on \emph{Corel5K} and \emph{MSRA} with varying data size.}\label{fig:time2}
\end{center}
\vspace{-0.2cm}
\end{figure}

\subsection{Efficiency Comparison}
It is crucial to study the efficiency of the compared MIML approaches, because our basic motivation is to develop a method that can work on large scale MIML data. All the experiments are performed on a machine with $16\times2.60$ GHz CPUs and 32GB main memory. Again, we first show the time cost of each algorithm on the six moderate-sized data sets in Figure \ref{fig:time}. Since the results on the two smallest data sets \emph{Letter Carroll} and \emph{Letter Frost} are similar, we take one of them as representative to save space. Obviously, our approach is the most efficient one on all the data sets. MIMLBoost is the most time-consuming one, followed by RankLossSIM and MIMLkNN.

The superiority of our approach is more distinguished on larger data sets. As shown in Figure \ref{fig:time2}, on \emph{Corel5K}, MIMLBoost failed to get result in 24 hours even with the smallest subset, while RankLossSIM can handle only 1000 examples. The time costs of existing methods increase dramatically as the data size increases. In contrast, MIMLfast takes only 1 minute even for the largest size in Figure \ref{fig:time2}(a). In Figure \ref{fig:time2}(b), on the largest \emph{MSRA} data, the superiority of MIMLfast is even more apparent. None of existing approaches can deal with more than 20K examples. In contrast, on data of 20,000 bags and 180,000 instances, MIMLfast is more than 100 times faster than the most efficient existing approach; when the data size becomes larger, none of existing approaches can return result in 24 hours, and MIMLfast takes only 12 minutes.

\begin{figure*}[!t]
\begin{center}\label{fig:keyIns}
\begin{minipage}{0.23\linewidth}
\includegraphics[width=\textwidth]{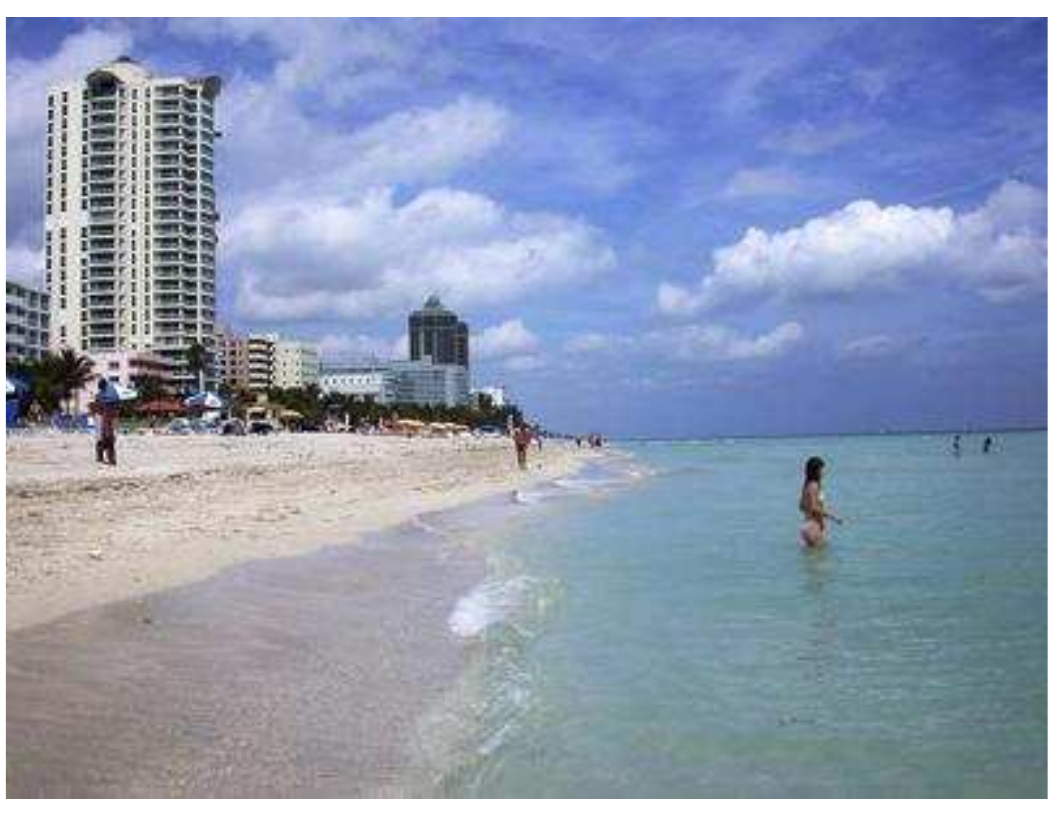}
\centering{original image}
\end{minipage}
\begin{minipage}{0.23\linewidth}
\vspace{-0.2em}
\includegraphics[width=\textwidth]{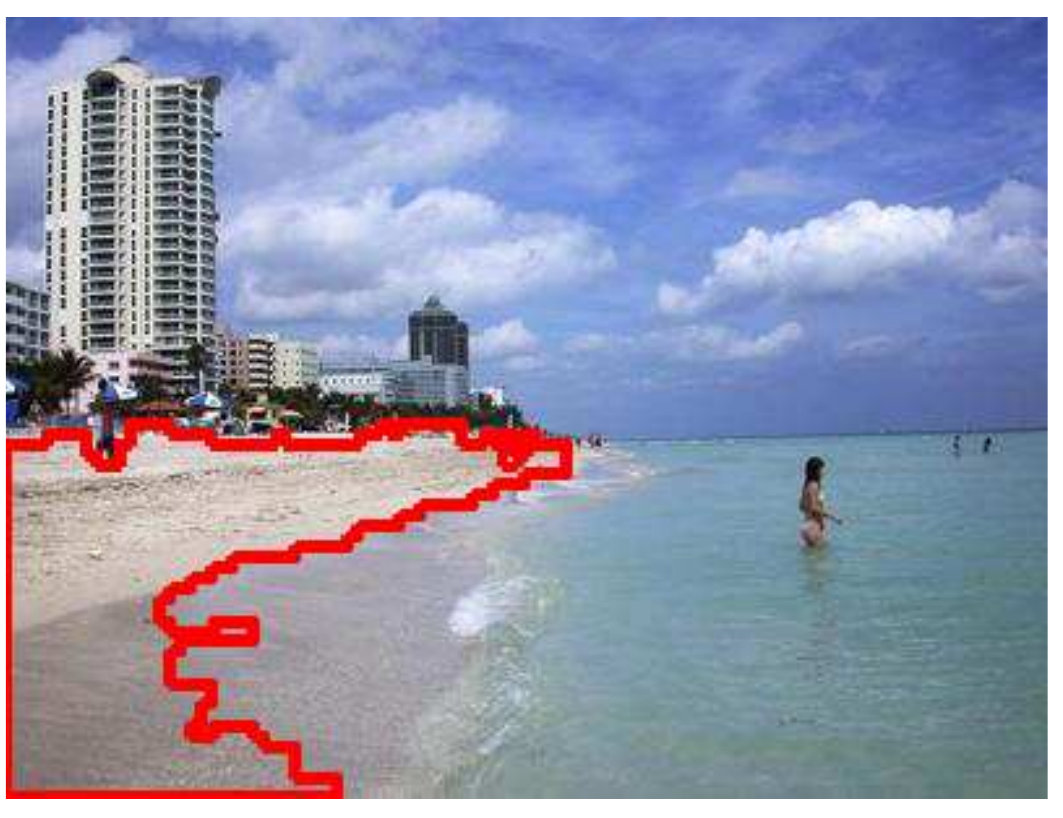}
\centering{label: beach}
\end{minipage}
\begin{minipage}{0.23\linewidth}
\includegraphics[width=\textwidth]{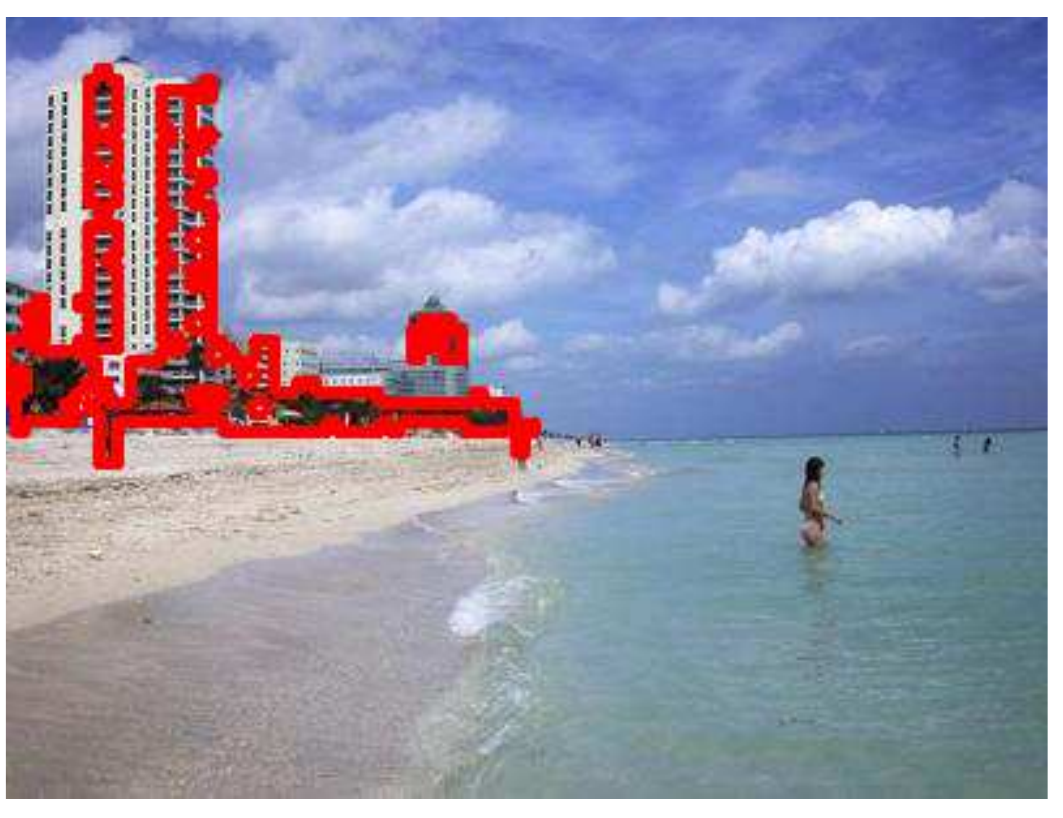}
\centering{label: building}
\end{minipage}
\begin{minipage}{0.23\linewidth}
\vspace{-0.2em}
\includegraphics[width=\textwidth]{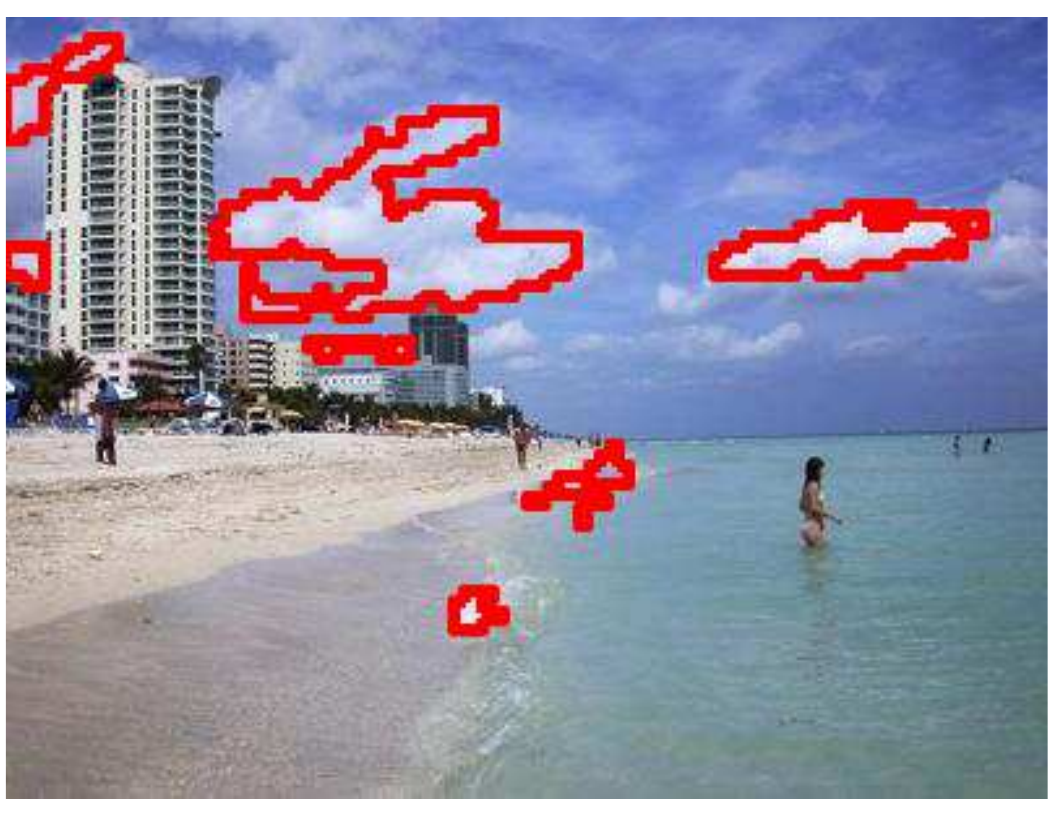}
\centering{label: cloud}
\end{minipage}\\
\begin{minipage}{0.23\linewidth}
$\ $
\end{minipage}
\begin{minipage}{0.23\linewidth}
\vspace{-0.2em}
\includegraphics[width=\textwidth]{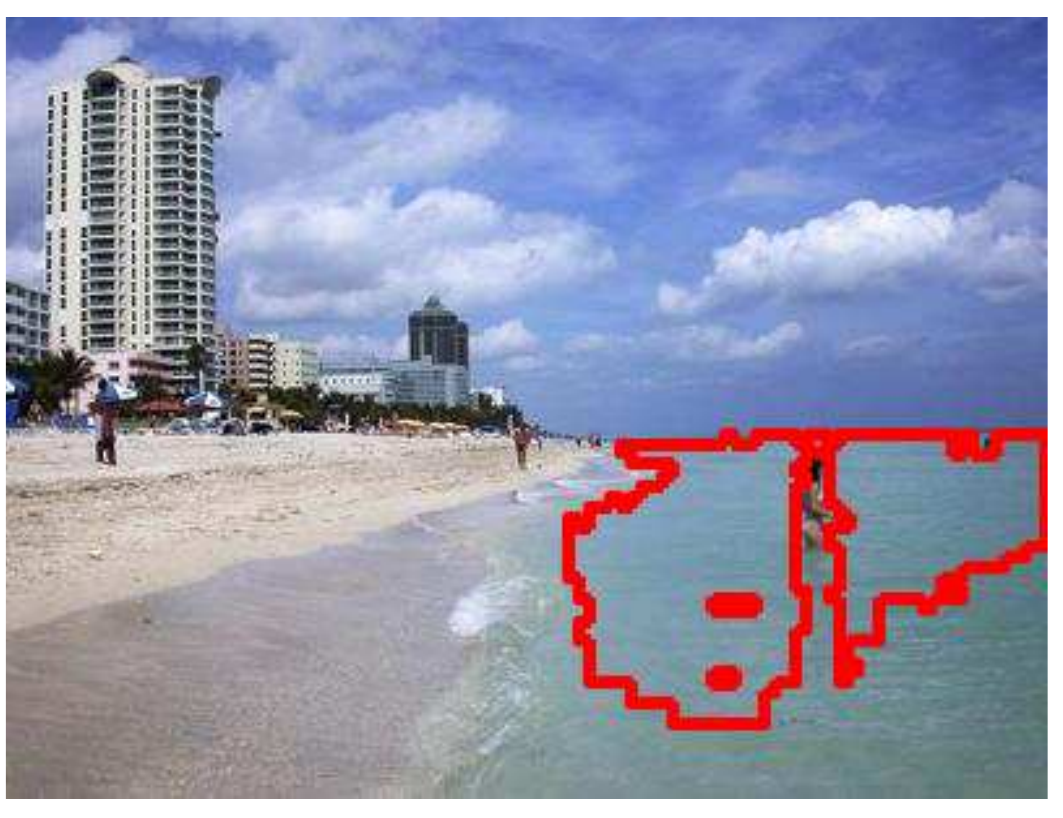}
\centering{label: sea}
\end{minipage}
\begin{minipage}{0.23\linewidth}
\includegraphics[width=\textwidth]{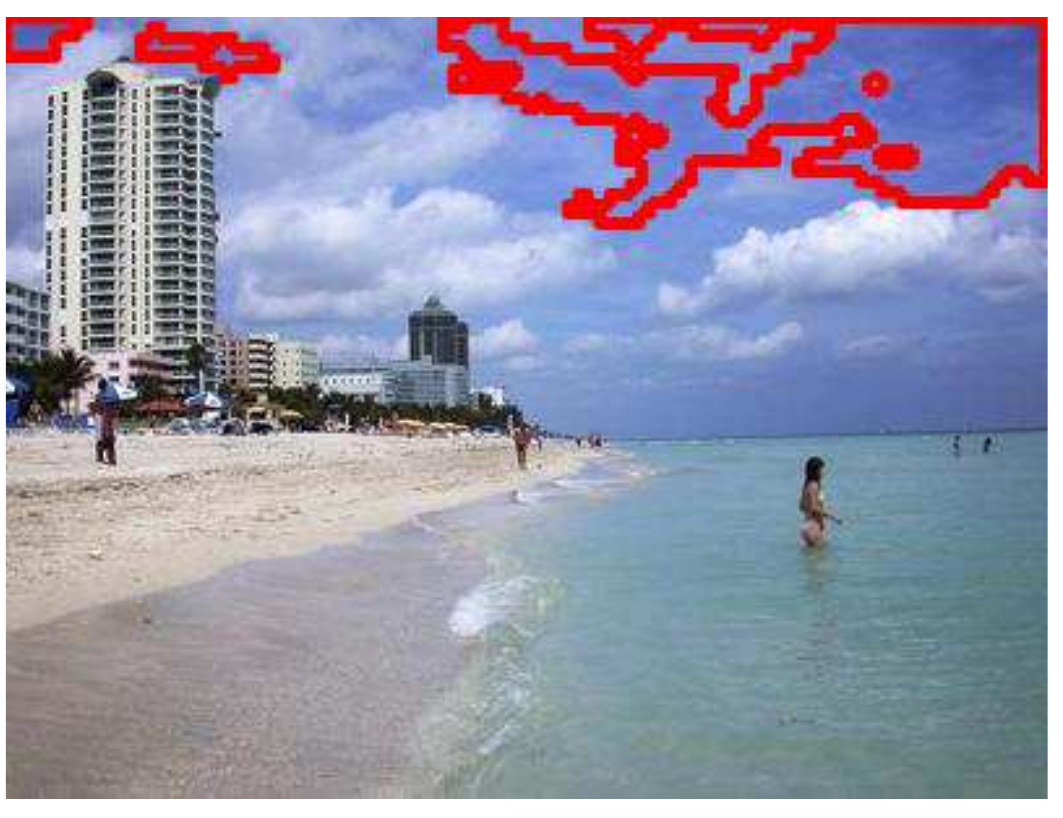}
\centering{label: sky}
\end{minipage}
\begin{minipage}{0.23\linewidth}
\vspace{-0.2em}
\includegraphics[width=\textwidth]{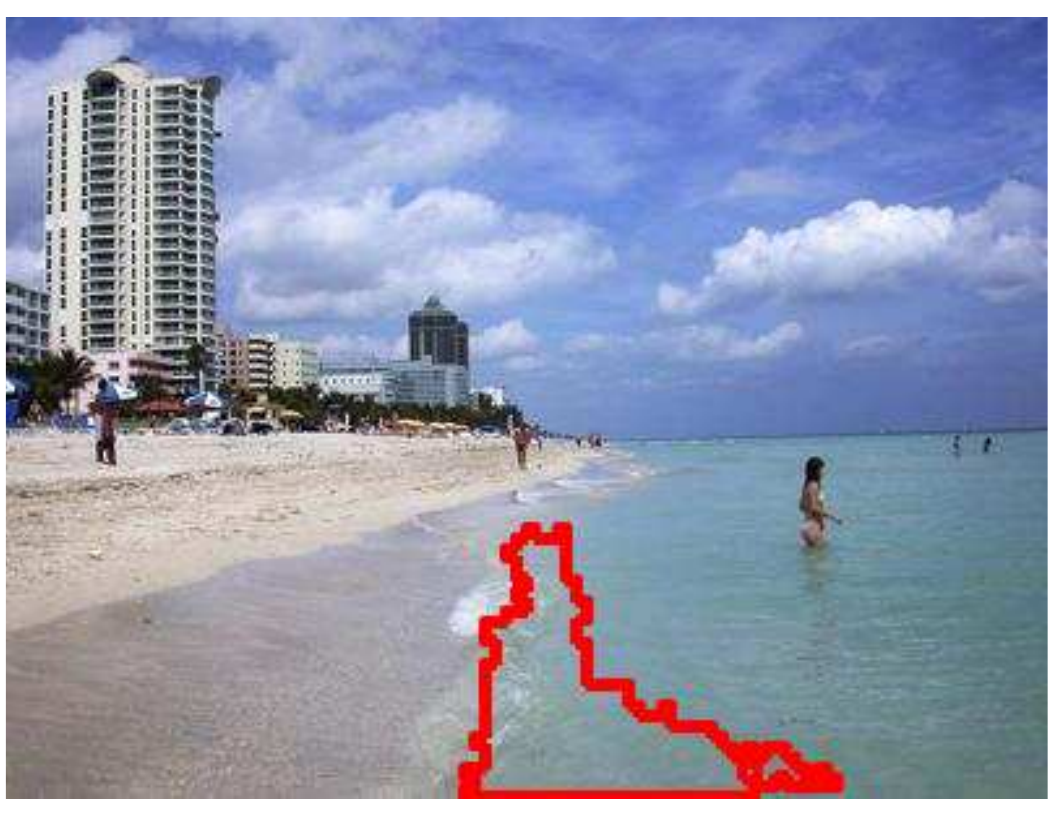}
\centering{label: water}
\end{minipage}\\
\begin{minipage}{0.23\linewidth}
\vspace{0.2em}
\includegraphics[width=\textwidth]{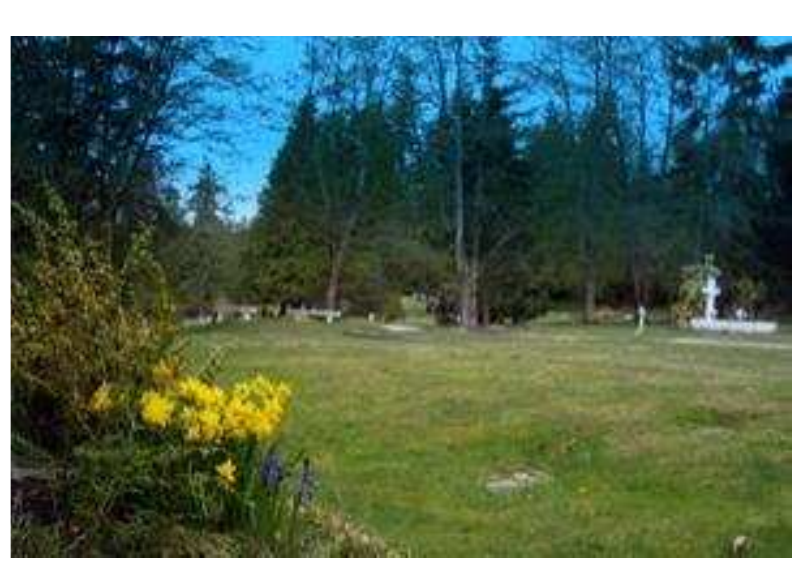}
\centering{original image}
\end{minipage}
\begin{minipage}{0.23\linewidth}
\includegraphics[width=\textwidth]{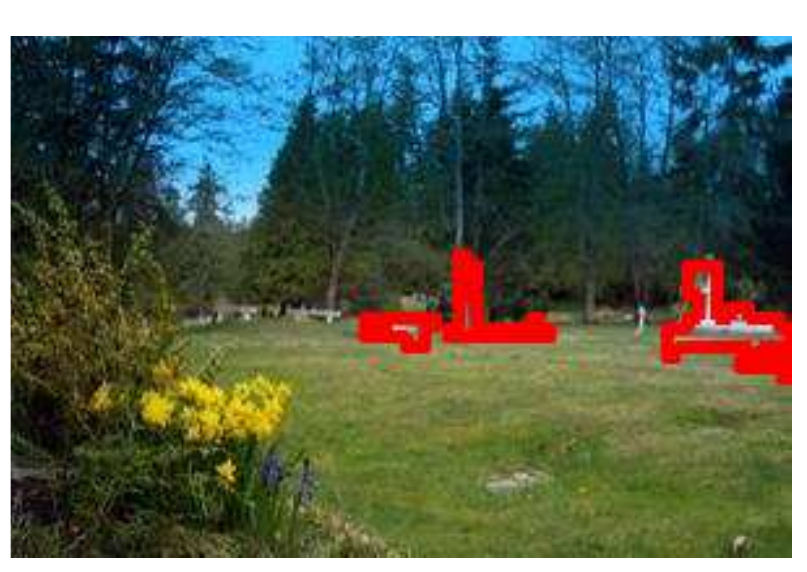}
\centering{label: cross}
\end{minipage}
\begin{minipage}{0.23\linewidth}
\includegraphics[width=\textwidth]{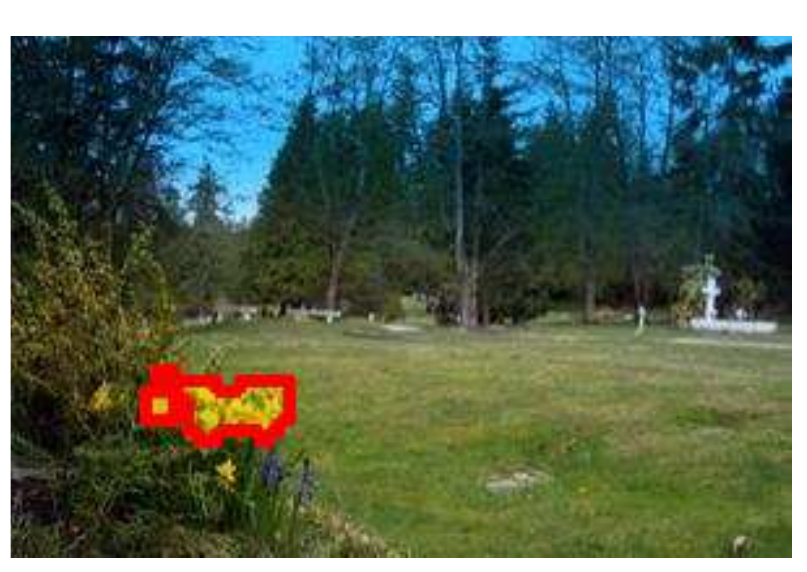}
\centering{label: flower}
\end{minipage}
\begin{minipage}{0.23\linewidth}
\vspace{0.2em}
\includegraphics[width=\textwidth]{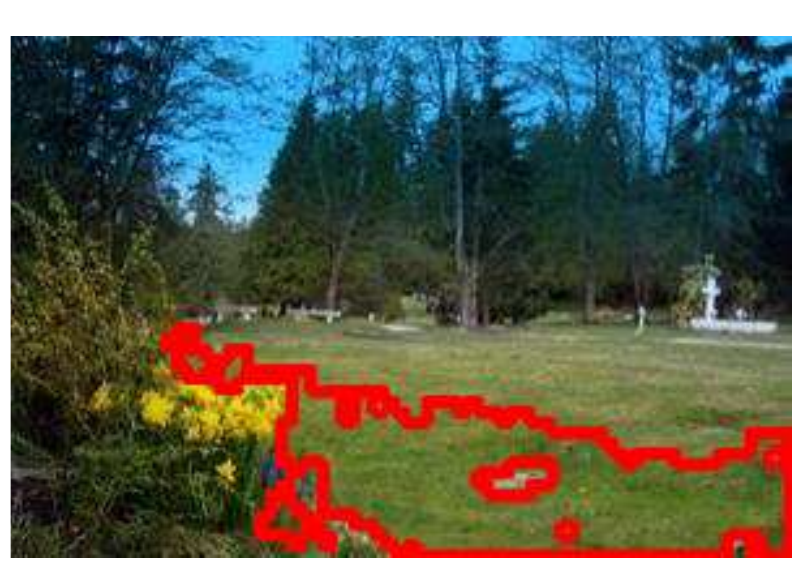}
\centering{label: landscape}
\end{minipage}\\
\begin{minipage}{0.23\linewidth}
$\ $
\end{minipage}
\begin{minipage}{0.23\linewidth}
\includegraphics[width=\textwidth]{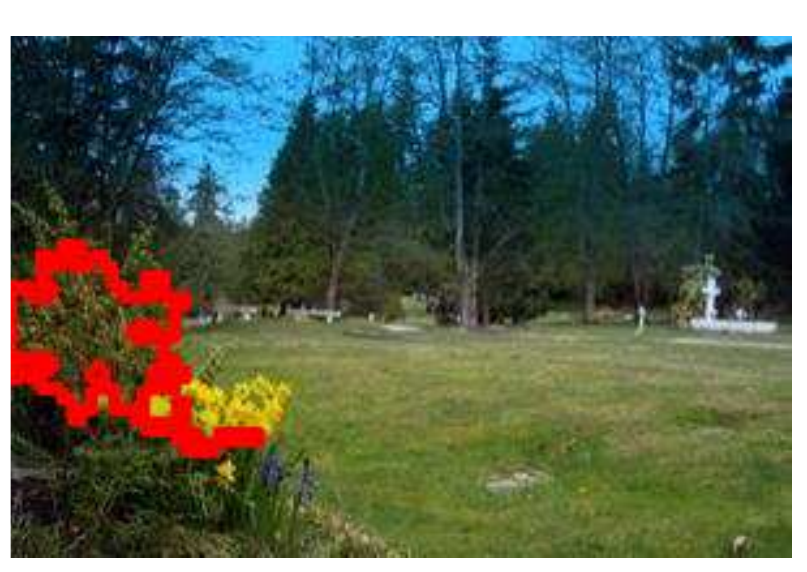}
\centering{label: leaf}
\end{minipage}
\begin{minipage}{0.23\linewidth}
\vspace{0.2em}
\includegraphics[width=\textwidth]{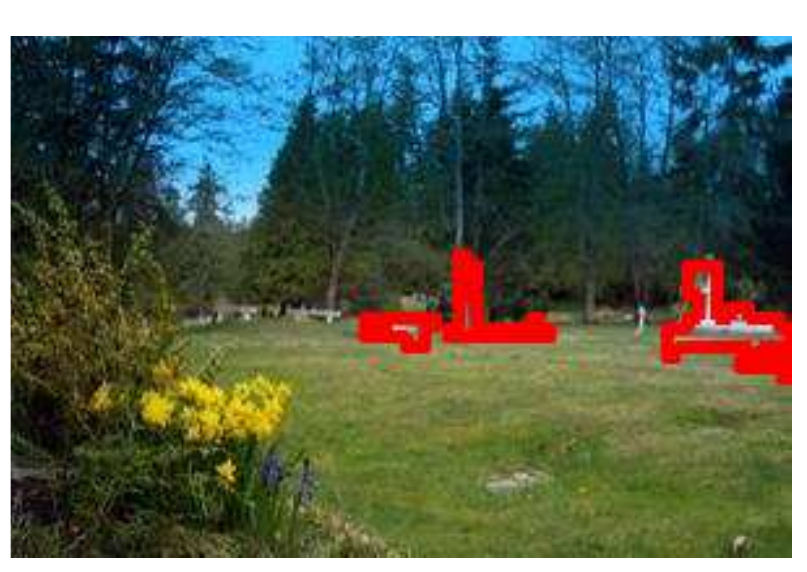}
\centering{label: sculpture}
\end{minipage}
\begin{minipage}{0.23\linewidth}
\vspace{0.2em}
\includegraphics[width=\textwidth]{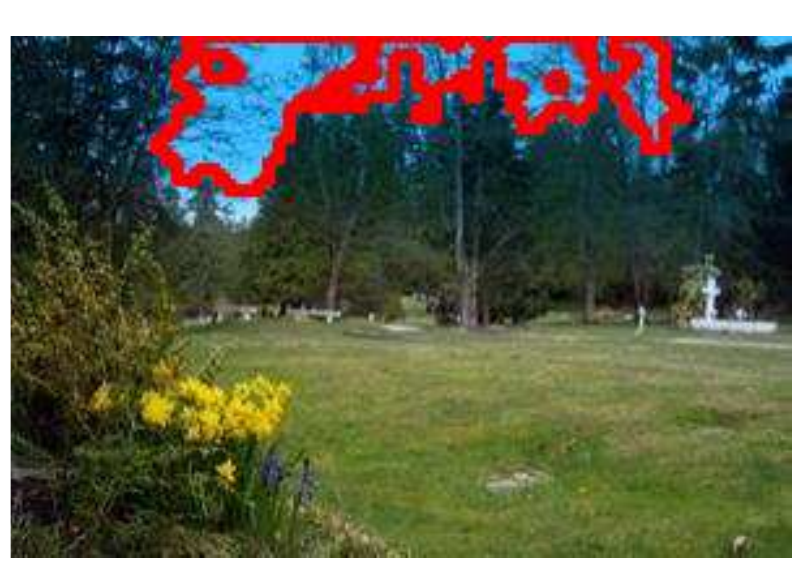}
\centering{label: sky}
\end{minipage}
\caption{Key instances identified by MIMLfast for each label. Image regions corresponding to key instances are highlighted with red contours.}\label{fig:keyIns}
\end{center}
%\vspace{-0.3cm}
\end{figure*}

\subsection{Key Instance Detection}
In MIML, a set of labels are assigned to a group of instances, and thus it is interesting to understand the relation between input patterns and output label semantics. Inspired by \cite{LHJZ12}, by assuming that each label is triggered by its most positive instance, our MIMLfast approach is able to identify the key instance for each label.

We first give an intuitive evaluation of the key instance detection of MIMLfast. On \emph{MSRA}, following \cite{LHJZ12}, we first partition each image into a set of patches with k-means clustering, and then extract an instance from each cluster. In Figure \ref{fig:keyIns}, we show two example images, and highlight the regions corresponding to the key instance detected by our approach for each label. Note that since the image regions are obtained by clustering, an instance may correspond to multiple regions in the same cluster rather than a single region. The results clearly show that MIMLfast can detect reasonable key instances for the labels.

We also evaluate the key instance detection accuracy quantitatively. On 4 of the 8 MIML data sets, i.e., \emph{Letter Carroll}, \emph{Letter Frost}, \emph{MSRC v2} and \emph{Bird Song}, the instance labels are available, and thus providing a test bed for key instance detection. Among the existing MIML methods, RankLossSIM and KISAR are able to detect key instance for each label, and will be compared with our approach. For MIMLfast and RankLossSIM, the key instance for a specific label is identified by selecting the instance with maximum prediction value on that label, while for KISAR, key instance is the one closest to the prototype of the label as in \cite{LHJZ12}. We examine the ground truth of the detected key instances and present the accuracies in Table \ref{table:keyInstance}. We can observe that KISAR is less accurate than the other two methods, probably because it does not build the model on the instance level, and detects key instance based on unsupervised prototypes. When compared with RankLossSIM, which is specially designed for instance annotation, our approach is more accurate on the two larger data sets, while comparable on \emph{Letter Carroll}, and slightly worse on \emph{Letter Frost}.

\begin{table}[!t]
\caption{Key instance detection accuracy (mean$\pm$std.). The best results are bolded.}
\label{table:keyInstance}
\begin{center}
\begin{tabular}{l|ccc}
\hline
data & MIMLfast & KISAR & RankLossSIM\\
\hline
\emph{LetterCarroll}\raisebox{1em}{} & \textbf{0.67$\pm$0.03} & 0.41$\pm$0.03  & \textbf{0.67$\pm$0.03}\\
\emph{LetterFrost} & 0.67$\pm$0.03 & 0.47$\pm$0.04 & \textbf{0.70$\pm$0.03}\\
\emph{MSRC v2}& \textbf{0.66$\pm$0.03}  & 0.62$\pm$0.03& 0.64$\pm$0.02\\
\emph{Bird Song} & \textbf{0.58$\pm$0.04} & 0.31$\pm$0.03 & 0.42$\pm$0.02\\
\hline
\end{tabular}
\end{center}
\vspace{-0.2cm}
\end{table}

\subsection{Sub-Concept Discovery}
To examine the effectiveness of sub-concept discovery, we run MIMLfast with varying number of sub-concepts on the two benchmark data sets: \emph{Scene} for image classification and \emph{Reuters} for text categorization. Table \ref{table:sub} presents the results with $K$ varying from 1 to 15 with step size of 5. For each value of $K$, we run 10-fold cross validation and report the average results as well as standard deviations. Note that $K$ is selected by cross validation on the training data in Section 3.2. As shown in Table \ref{table:sub}, compared with neglecting the sub-concepts ($K=1$), the exploitation of sub-concepts is helpful ($K=5$, 10 and 15 are all better than $K=1$). When the $K$ gets larger, the difference between results with different $K$ values is not very significant. This may owe to that if we set a $K$ value larger than what is really needed, some sub-concepts might capture no examples, and thus a overly-large $K$ will not make the performance degenerate too much, although it might hamper the efficiency.

We further examine the sub-concepts discovered by MIMLfast. We take the \emph{Scene} data set as an illustration and show some example images of the top-four sub-concepts discovered for the label \emph{sea} in Figure \ref{fig:subconcept}. It is interesting to see that these four sub-concepts are with reasonable but different perceptions: the first sub-concept corresponds to sea with beach and blue sky, the second sub-concept corresponds to big wave in the sea, etc.

\begin{table}[!t]
\caption{Results (mean$\pm$std.) obtained by identifying different numbers of sub-concepts. The best performance and its comparable results based on paired $t$-tests at 95\% significance level are bolded.}
\label{table:sub}%\small
\begin{center}
\begin{tabular}{l|cccc}
\hline
$K$ & 1 & 5 & 10 & 15\\
\hline
\multicolumn{2}{l}{\emph{Scene}}\raisebox{1em}{$\ $}\\
\hline
\textsf{hamming loss} $^{\downarrow}$\raisebox{1em}{} & .191$\pm$.011 & \textbf{.186$\pm$.009} & \textbf{.182$\pm$.014} & \textbf{.181$\pm$.011}\\
\textsf{one error} $^{\downarrow}$ & .366$\pm$.038 & .354$\pm$.026 & \textbf{.338$\pm$.030} & \textbf{.344$\pm$.031}\\
\textsf{coverage} $^{\downarrow}$ & .224$\pm$.018 & .213$\pm$.015 & \textbf{.202$\pm$.017} & .210$\pm$.014\\
\textsf{ranking loss} $^{\downarrow}$ & .209$\pm$.020 & .196$\pm$.017 & \textbf{.184$\pm$.018} & .192$\pm$.016\\
\textsf{average precision} $^{\uparrow}$ & .754$\pm$.023 & .764$\pm$.018 & \textbf{.777$\pm$.020} & \textbf{.769$\pm$.019}\\
\hline
\multicolumn{2}{l}{\emph{Reuters}}\raisebox{1em}{$\ $}\\
\hline
\textsf{hamming loss} $^{\downarrow}$\raisebox{1em}{} & .027$\pm$.008 & .026$\pm$.006 & \textbf{.025$\pm$.006} & \textbf{.025$\pm$.007}\\
\textsf{one error} $^{\downarrow}$ & .042$\pm$.013 & .040$\pm$.009 & \textbf{.037$\pm$.007} & \textbf{.040$\pm$.010}\\
\textsf{coverage} $^{\downarrow}$ & .036$\pm$.007 & \textbf{.035$\pm$.006} & \textbf{.034$\pm$.006} & \textbf{.035$\pm$.007}\\
\textsf{ranking loss} $^{\downarrow}$ & .015$\pm$.006 & \textbf{.014$\pm$.005} & \textbf{.013$\pm$.005} & \textbf{.014$\pm$.006}\\
\textsf{average precision} $^{\uparrow}$ & .972$\pm$.010 & .974$\pm$.007 & \textbf{.976$\pm$.006} & \textbf{.974$\pm$.008}\\
\hline
\end{tabular}
\end{center}
\end{table}

\begin{figure}[!t]
\begin{center}
\begin{minipage}{0.2\linewidth}
\includegraphics[width=\textwidth]{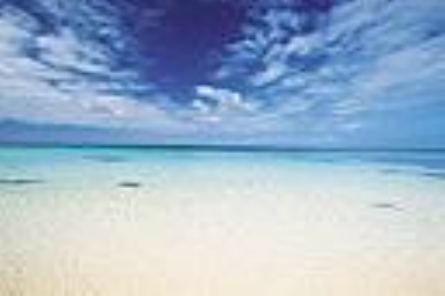}
\end{minipage}
\begin{minipage}{0.2\linewidth}
\includegraphics[width=\textwidth]{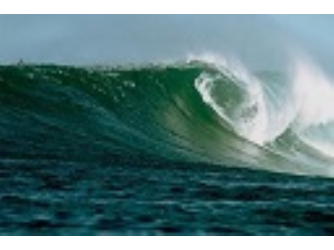}
\end{minipage}
\begin{minipage}{0.2\linewidth}
\includegraphics[width=\textwidth]{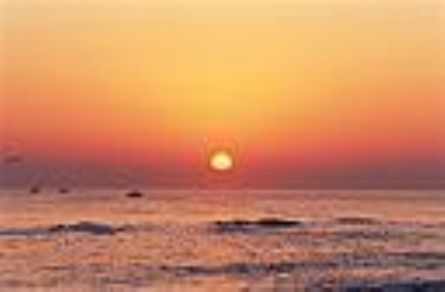}
\end{minipage}
\begin{minipage}{0.2\linewidth}
\includegraphics[width=\textwidth]{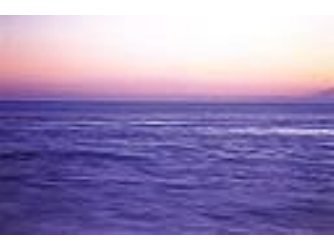}
\end{minipage}
\begin{minipage}{0.2\linewidth}
\includegraphics[width=\textwidth]{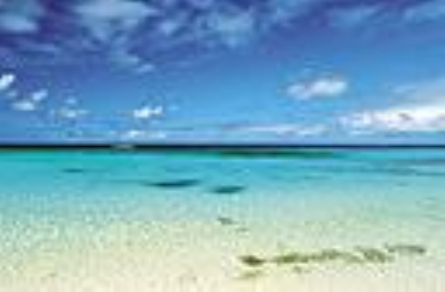}
\end{minipage}
\begin{minipage}{0.2\linewidth}
\includegraphics[width=\textwidth]{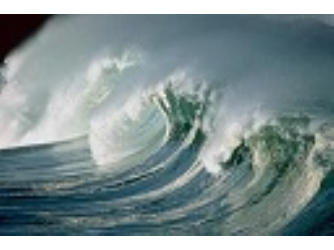}
\end{minipage}
\begin{minipage}{0.2\linewidth}
\includegraphics[width=\textwidth]{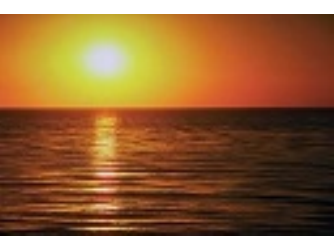}
\end{minipage}
\begin{minipage}{0.2\linewidth}
\includegraphics[width=\textwidth]{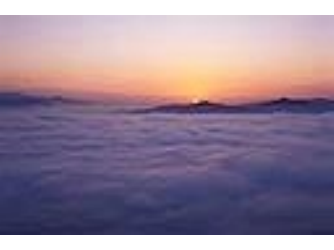}
\end{minipage}
\begin{minipage}{0.2\linewidth}
\includegraphics[width=\textwidth]{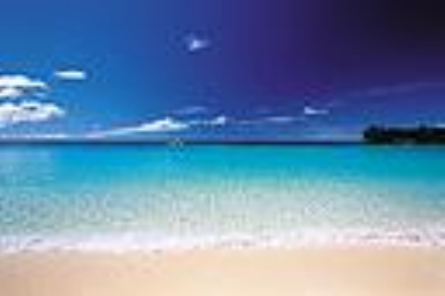}
\end{minipage}
\begin{minipage}{0.2\linewidth}
\includegraphics[width=\textwidth]{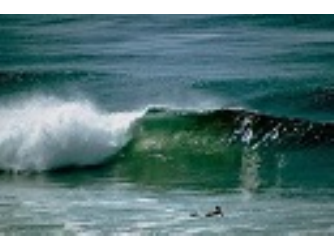}
\end{minipage}
\begin{minipage}{0.2\linewidth}
\includegraphics[width=\textwidth]{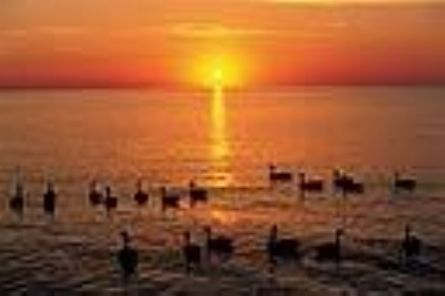}
\end{minipage}
\begin{minipage}{0.2\linewidth}
\includegraphics[width=\textwidth]{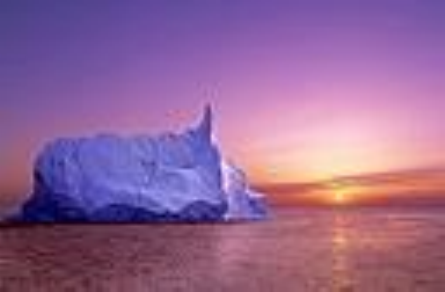}
\end{minipage}
\caption{Example images of different sub-concepts identified for label \emph{sea}, where one column corresponds to one sub-concept.}\label{fig:subconcept}
\end{center}
\vspace{-0.2cm}
\end{figure}

\begin{table}[!t]
\caption{Comparison results (mean$\pm$std.) of MIMLfast with two variants (V1 and V2). The best performance and its comparable results based on paired $t$-tests at 95\% significance level are bolded.}
\label{table:variants}%\small
\begin{center}
\begin{tabular}{l|ccc}
\hline
 & MIMLfast & V1 & V2\\
\hline
\multicolumn{2}{l}{\emph{Scene}}\raisebox{1em}{$\ $}\\
\hline
\textsf{hamming loss} $^{\downarrow}$ & \textbf{.188$\pm$.009} & .211$\pm$.009 & .196$\pm$.012\\
\textsf{one error} $^{\downarrow}$ & \textbf{.351$\pm$.023} & .409$\pm$.023 & \textbf{.358$\pm$.030}\\
\textsf{coverage} $^{\downarrow}$ & \textbf{.207$\pm$.012} & .239$\pm$.011 & \textbf{.208$\pm$.014}\\
\textsf{ranking loss} $^{\downarrow}$ & \textbf{.189$\pm$.014} & .228$\pm$.013 & \textbf{.192$\pm$.016}\\
\textsf{avg. precision} $^{\uparrow}$ & \textbf{.770$\pm$.015} & .730$\pm$.014 & \textbf{.767$\pm$.018}\\
\hline
\multicolumn{2}{l}{\emph{Reuters}}\raisebox{1em}{$\ $}\\
\hline
\textsf{hamming loss} $^{\downarrow}$ & \textbf{.028$\pm$.004} & .038$\pm$.004 & .035$\pm$.003\\
\textsf{one error} $^{\downarrow}$ & \textbf{.044$\pm$.008} & .060$\pm$.011 & .046$\pm$.010\\
\textsf{coverage} $^{\downarrow}$ & \textbf{.035$\pm$.004} & .038$\pm$.005 & \textbf{.035$\pm$.004}\\
\textsf{ranking loss} $^{\downarrow}$ & \textbf{.014$\pm$.004} & .019$\pm$.004 & \textbf{.015$\pm$.003}\\
\textsf{avg. precision} $^{\uparrow}$ & \textbf{.972$\pm$.005} & .963$\pm$.007 & \textbf{.971$\pm$.006}\\
\hline
\end{tabular}
\end{center}
\vspace{-0.2cm}
\end{table}

\subsection{Comparison with Variants}
To further examine how MIMLfast works, we study two variants, V1 and V2. V1 gives up $W_0$ in Eq. \ref{eq:model} and directly learns a linear model for each label. It is constructed to examine whether learning the shared space is helpful. V2 simply selects the top $r$ labels as relevant ones, where $r$ is the average number of relevant labels on the training data. It is constructed to examine whether the dummy label provides a good separation of relevant and irrelevant labels.

Table \ref{table:variants} shows the results on the two benchmark data sets. V1 is significantly worse than MIMLfast on all criteria, implying that learning the shared space for all the labels is better than learning each label independently. On \textsf{hamming loss}, MIMLfast achieves significantly better performance than V2, while on the other four criteria, they achieve comparable performances, implying that the use of dummy label does not affect the rank of the labels but providing a reasonable separation of relevant and irrelevant labels.

\section{Related Work}
Many MIML approaches were proposed during the past few years. For example, MIMLSVM \cite{ZZ07} degenerated the MIML problem into single-instance multi-label tasks to solve. MIMLBoost \cite{ZZ07} degenerated MIML to multi-instance single-label learning. A generative model for MIML was proposed by Yang et al. \cite{YZH09}. Nearest neighbor and neural network approaches for MIML were proposed in \cite{Z10} and \cite{ZW09}, respectively. Zha et al. \cite{ZHMWQW08} proposed a hidden conditional random field model for MIML image annotation. Briggs et al. \cite{BFR12} proposed to optimize ranking loss for MIML instance annotation. In \cite{LHJZ12}, the authors tried to discover what patters trigger what labels in MIML learning by constructing a prototype for each label with clustering. Existing MIML approaches achieved success in many applications, most with moderate-sized data owing to the high computational load. To handle large-scale data, MIML approaches with high efficiency are demanded.

In \cite{WBU11}, a similar technique was used to optimize WARP loss for image annotation; however, it dealt with single-instance single-label problem, which is quite different from our MIML problem. In \cite{ZZHL12}, an approach of discovering sub-concepts for complicated concepts was proposed based on clustering. However, it was focused on  single label learning, quite different from our MIML task. Moreover, MIMLfast exploits label information and discovers sub-concepts using supervised model rather than heuristic clustering.

\section{Conclusion}
MIML is a framework for learning with complicated objects, and has been proved to be effective in many applications. However, existing MIML approaches are usually too time-consuming to deal with large scale problems. In this paper, we propose the MIMLfast approach to learn with MIML examples fast. On one hand, efficiency is highly improved by optimizing the approximated ranking loss with SGD based on a two level linear model; on the other hand, effectiveness is achieved by exploiting label relations in a shared space and discovering sub-concepts for complicated labels. Moreover, our approach can naturally detect key instance for each label, and thus providing a chance to discover the relation between input patterns and output label semantics. In the future, we will try to optimize other loss functions rather than ranking loss. Also, larger scale problems will be studied.

\bibliographystyle{plain}
\bibliography{miml}

\end{document}